\documentclass{alt2026} % submission template (anonymized)
\newcommand{\KL}{\mathrm{KL}}
\newcommand{\E}{\mathbb{E}}
\newcommand{\Ec}{\mathcal{E}}

\newcommand{\R}{\mathbb{R}}
\newcommand{\Pc}{\mathcal{P}}
\newcommand{\Lc}{\mathcal{L}}
\newcommand{\Mcal}{\mathcal{M}_+}

\newcommand{\dd}{\,\mathrm{d}}
\newcommand{\EVaR}{\mathrm{EVaR}}
\newcommand{\ERM}{\mathrm{ERM}}

\newcommand{\Pcal}{\mathcal{P}}

\newcommand{\xmax}{x_{\max}}

\DeclareMathOperator*{\argmax}{arg\,max}
\DeclareMathOperator*{\argmin}{arg\,min}
\DeclareMathOperator{\supp}{supp}  % nicer than \text{supp}

\newcommand{\KLiU}{\mathrm{KL}_{\inf}^{\mathrm{U}}}
\newcommand{\CVaR}{\mathrm{CVaR}}

\newcommand{\KLiL}{\mathrm{KL}_{\inf}^{\mathrm{L}}}
\newcommand{\Mc}{\mathcal{M}}

\definecolor{evcol}{HTML}{007F7F}

\usepackage{tikz,xcolor}
\usepackage{standalone}
\usetikzlibrary{arrows.meta,calc,decorations.pathreplacing}
\title[Short Title]{EVaR-Optimal Arm Identification in Bandits}
\usepackage{times}
\newcommand{\COMMENTSWITCH}{1} % <- change to 2 to remove all comments

\makeatletter
\ifnum\COMMENTSWITCH=2
  \newcommand{\cmt}[1]{\textcolor{red}{#1}}
\else
  \newcommand{\cmt}[1]{}
\fi
\makeatother
\ifnum\COMMENTSWITCH=2
  \newcommand{\fix}[1]{\textcolor{teal}{#1}}
\else
  \newcommand{\fix}[1]{\textcolor{black}{#1}}
\fi
\makeatother
\altauthor{%
 \Name{Mehrasa Ahmadipour} \Email{Mehrasa.ahmadipour@ens-lyon.fr}\\
 \addr UMPA,
ENS de Lyon
 \AND
 \Name{Aurélien Garivier} \Email{aurelien.garivier@ens-lyon.fr}\\
 \addr 
UMPA/LIP
ENS de Lyon
}

\begin{document}

\maketitle

%\tableofcontents

\begin{abstract}%
  We study the fixed-confidence best arm identification (BAI) problem within the multi-armed bandit (MAB) framework under the Entropic Value-at-Risk (EVaR) criterion. Our analysis considers a nonparametric setting, allowing for general reward distributions bounded in $[0,1]$.
\fix{  This formulation addresses the critical need for risk-averse decision-making in high-stakes environments, such as finance, moving beyond simple expected value optimization.}
We propose a $\delta$-correct, Track-and-Stop based algorithm and derive a corresponding lower bound on the expected sample complexity, which we prove is asymptotically matched. The implementation of our algorithm and the characterization of the lower bound both require solving a complex convex optimization problem and a related, simpler non-convex one.
\cmt{
Finally, we report numerical exper}
\end{abstract}
%%%%%%%%%
%%%%%%%%%
\section{Introduction}
%{\color{red}  Use $\rho:=-\log(1-\alpha)$} 
\label{Intro}
Risk-sensitive decision-making is crucial in different fields such as finance, insurance, operations, and safety-critical systems. This has spurred the development and analysis of a broad array of risk measures, including mean-variance \cite{markowitz1952portfolio}, Conditional Value-at-Risk (CVaR) \cite{artzner1999coherent,rockafellar2000cvar}, spectral risk measures \cite{acerbi2002spectral}, cumulative prospect theory \cite{tversky1992cumulative}, utility-based shortfall risk \cite{artzner1999coherent,hu2018utility}, and more recently, Entropic Value-at-Risk (EVaR) \cite{ahmadi2011information}. Applying these criteria to sequential decision-making, where outcomes are learned through repeated interaction, remains a key challenge.

The multi-armed bandit (MAB) framework provides a formal model for sequential decision-making under uncertainty. In the fixed-confidence setting, the goal is to identify the best arm with a guaranteed probability of being correct. The primary performance metric in this setup is sample complexity (the expected number of samples required to make a decision) which we aim to minimize.
Beyond maximizing the expected reward, risk-aware bandit formulations have received increasing attention. Prior work has explored a range of objectives and settings, including mean-variance bandits and VaR-based objectives \cite{vakili2015meanvariance}, and risk-averse strategies in (sub-)Gaussian environments \cite{sani2012riskaversion}. Formulations centered on CVaR have been particularly well-studied, spanning best-arm identification (BAI) \cite{agrawal2021BAIcvar}, regret minimization with CVaR/VaR constraints \cite{agrawalregret1}, and Thompson Sampling-based approaches for regret minimization \cite{baudry2021optimal}.

Despite extensive research on risk-aware bandits, the more recent notion EVaR has received comparatively little attention, despite a recent surge of interest due to its particular role in Markov Decision Processes (see e.g. \cite{martheGV23beyond}). One notable study is the pioneering contribution \cite{maillard2013robust}, which proposed an algorithm for regret minimization under Entropic Risk Measure (ERM), a closely related quantity. However, to the best of our knowledge, the best arm identification (BAI) problem for EVaR in the fixed-confidence setting remains unexplored.  This paper is the first to address this gap.

\paragraph{Setting.}
We consider a $K$-armed bandit with unknown reward (or loss) laws $(\nu_i)_{i\in[K]}$ drawn from a class $\mathcal{P}$, and fix a risk level $\alpha\in(0,1)$. 
%For each arm $i$, we measure risk via the Entropic Value-at-Risk
% \[
% \EVaR_\alpha(X_i)\;:=\;\inf_{\lambda>0}\,\frac{1}{\lambda}\Big(\log \E_{\nu_i}\!\big[e^{\lambda X_i}\big]-\log \alpha\Big).
% \]
For a fixed risk parameter $\delta\in(0,1)$, our goal is to design a stopping time $\tau_\delta$ and a recommendation $\hat{\imath}$ such that, with probability at least $1-\delta$, $\hat{\imath}\in \arg\min_{i\in K}\EVaR_\alpha(X_{\nu_i})$, while minimizing the expected sample size $\E_\nu[\tau_\delta]$. The key departure from classical mean-based best-arm identification is the EVaR feasibility constraint $\EVaR_\alpha(X_{\nu_i})$, which fundamentally changes both the information structure and the optimal allocation of samples.
 
The classical asymptotically optimal algorithms for BAI in MAB are guided by an information-theoretic lower bound \cite{garivier2016optimal}. This bound is characterized by an optimization problem that takes the true arm distributions as inputs and solves for the optimal sampling proportions. These algorithms operate sequentially by using the empirical distributions gathered so far as a plug-in estimate for the true distributions, which in turn guides the sampling strategy.

We follow this high-level structure, but adapting it to the performance metric considered here introduces unique technical challenges. Following the spirit of the analysis for CVaR-BAI by \citet{agrawal2021BAIcvar}, we now introduce two key functionals that are central to addressing the technical challenges arising in the EVaR case.

For two probability measures $\eta_1,\eta_2\in\Pc(\R)$, Kullback–Leibler divergence is defined as $\KL(\eta_1\Vert\eta_2)$ $:=$ $\int \log\!\Big(\frac{d\eta_1}{d\eta_2}\Big)\,d\eta_1$, or $+\infty$ when  $\eta_1\not\ll \eta_2$.
Our analysis centers on two information distances, which are KL-projections of a measure \(\eta\) onto a feasible set defined by an \(\EVaR_\alpha(\cdot)\) constraint formally defined in Equation~\eqref{def:evar:mgf} below:
\begin{equation}
\begin{aligned}
\label{def:KLU/L}
\KLiU(\eta, \nu):=\min_{\substack{\kappa \in \Pcal([0,1]): 
\\ \EVaR_\alpha(X_\kappa) \ge \nu}}
\quad & \KL(\eta\Vert \kappa),
%\\
%\text{s.t.}\quad & \E_Q[X]\ \ge\ \nu,
%\\
%& \KL(Q\Vert \kappa)\ \le\ \rho.
\hspace{2cm}
\KLiL(\eta, \nu) :=\min_{\substack{\kappa \in \Pcal([0,1]): 
\\ \EVaR_\alpha(X_\kappa) \le \nu}}
\quad & \KL(\eta\Vert \kappa).
\end{aligned}
\end{equation}
Functionals of this type already appeared in \cite{burnetas96optimal}; they have since been adapted to identify the best mean in a nonparametric setting \cite{agrawalBAImean} and recently under a CVaR constraint \cite{agrawal2021BAIcvar} which  which is most relevant to our work.

These functionals \(\KLiU\) and \(\KLiL\) are asymmetric and require separate analyses for the upper and lower forms. They drive our contributions: we first use them to characterize the sample-complexity lower bound (Section~\ref{sec:lowerbound}); the same functionals then underpin our \(\delta\)-correct algorithm (Section~\ref{sec:alg:evar}); and finally, their structure enables a sharp performance analysis, yielding the requisite concentration inequalities via supermartingale methods.
{ Studying KL projections is essential because they provide the geometric and algorithmic primitives that let us extend our analysis to other risk-aware decision problems.}

%\subsection{Our model and Contribution} 
\textbf{Organization of the paper: }%(Subsection~\ref{subsec:mot_def}).  
\fix{We continue this section by motivation on EVaR with its useful definitions and finish it with an overviewing  the technical challenges.}
 Section~\ref{sec:assum} gathers the standing assumptions and lemmas.
In Section~\ref{sec:lowerbound}, we formalize the bandit model, derive an instance-dependent lower bound on sample complexity, and identify the KL–projection quantities that appear in the bound; we also introduce a characteristic information term based on KL projections and develop the supporting theory.
Section~\ref{sec:alg:evar} develops a tractable optimization reformulation and presents a $\delta$-correct algorithm built upon it; we prove asymptotic optimality-under stated conditions, its sample complexity matches the lower bound.
\cmt{ Section~\ref{sec:experiment} reports numerical experiments.On the empirical side, we compare our approach against a plug-in EVaR estimator (and standard LUCB-style baselines), observing substantial sample-efficiency gains especially under heavy-tailed.
}

We conclude this section by motivating the EVaR-based objective, presenting equivalent formulations, and clarifying its links to standard risk measures.
Overall, the techniques introduced here contribute a principled pathway for incorporating entropic risk criteria into bandit learning.

\subsection*{What is the EVaR? Why should we care?}
\label{subsec:mot_def}
The Entropic Value-at-Risk (EVaR) is a powerful risk measure that remains under-explored in sequential decision-making contexts like multi-armed bandits. 
 Its importance stems from two key properties: it serves as the tightest coherent upper bound on Conditional Value-at-Risk (CVaR) which is a cornerstone risk measure with numerous applications, and it is intrinsically linked to the Entropic Risk Measure (ERM). We motivate our system model here by explaining these links in this subsection. For notational simplicity, in this subsection we view risk measures as functions of random variables; in the BAI formulation, we will instead regard them as functionals of the underlying distributions.

For a random variable \(X\) representing loss and a confidence level \(\alpha \in (0,1)\), the {Conditional Value-at-Risk (CVaR)} at tail probability \(1-\alpha\) is defined as:
\[
    \CVaR_{\alpha}(X) = \inf_{v \in \mathbb{R}} \left\{ v + \frac{1}{1-\alpha} \E\big[(X - v)^+\big] \right\},
\]
where \((z)^+ = \max(0, z)\) denotes the positive part of \(z\).
While CVaR is widely used, EVaR offers significant analytical and computational advantages. Defined via its exponential-moment form, EVaR is the tightest Chernoff-style bound on (C)VaR:
% EVaR is the tightest Chernoff-based coherent upper bound on CVaR. It has an \emph{exponential-moment form} tied directly to moment-generating functions and concentration arguments:
\begin{equation}
    \EVaR_\alpha(X) = \inf_{z>0} \frac{1}{z} \left( \log \E[e^{zX}] +\rho \right).
    \label{def:evar:mgf}
\end{equation}
This  Moment-Generating Function (MGF) formulation connects directly to moment-generating functions and concentration inequalities. More revealingly, EVaR has a compelling distributionally robust interpretation as the worst-case expectation over an ambiguity set defined by a Kullback-Leibler (KL) divergence constraint:
\begin{equation}
    \EVaR_\alpha(X) = \sup_{Q\ll P:\,\ \KL(Q,P) \le \rho} \E_Q[X],
    \label{def:evar:KL}
\end{equation}
where \(\rho:=-\log(1-\alpha)\).
This link to relative entropy makes EVaR amenable to smooth, convex optimization. In contrast, the dual representation of CVaR imposes a hard density ratio cap \(0 \le \frac{dQ}{dP} \le (1-\alpha)^{-1}\) a structure that is often more challenging to work with \cite{ahmadi2011information}, \cite{acerbi2002spectral}.
Since \(\EVaR_\alpha(X) \ge \CVaR_{\alpha}(X)\) at the same level \cite[Theorem~6]{ahmadi2011information}, EVaR-feasible solutions provide conservative certificates for CVaR constraints.
\begin{figure}[!h]
    \centering
    \includegraphics[width=0.65\linewidth]{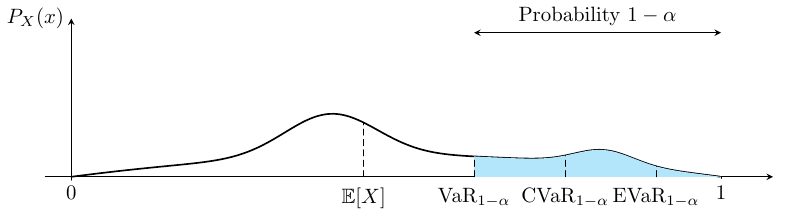}
    \caption{\fix{Schematic density on $[0,1]$ with right tail mass $1-\alpha$ shaded. Markers show $\mathbb{E}[X]$, $\mathrm{VaR}_\alpha(X)$, $\mathrm{CVaR}_\alpha(X)$, and $\mathrm{EVaR}_\alpha(X)$.}}
    \label{fig:placeholder}
\end{figure}

%{\color{red} To Rewrite: Entropic Risk Measure vs Entropic Value at risk }
It is also crucial to distinguish EVaR from the closely related Entropic Risk Measure (ERM). While mathematically linked, they represent fundamentally different approaches to risk. The {ERM} is defined as $\ERM_z(X) := \frac{1}{z} \log\left(\E[\exp(zX)]\right)$.
Here, \(z>0\) is a {\it{fixed, user-defined parameter} } representing risk aversion. Its distributionally robust interpretation reveals that it applies a {soft penalty} on the KL-divergence from the nominal model \(P\):
\[
    \ERM_z(X) = \sup_{Q\ll P}\left\{\E_Q[X] - \frac{1}{z}\KL({Q}||{P})\right\}.
\]
A larger \(z\) reduces the penalty, allowing for more adversarial distributions \(Q\) and thus a more conservative risk estimate. In contrast, EVaR uses a {hard constraint} on the KL-divergence, where the confidence level \(\alpha\) directly controls the radius of the uncertainty set.
% \[
%     \EVaR_\alpha(X) = \sup_{Q\ll P:\, \KL(Q, P)\le -\log(1-\alpha)} \E_Q[X].
% \]
The mathematical connection makes this "soft vs. hard" distinction explicit. For EVaR, \(z\) is not a fixed parameter but an internal variable over which an optimization is performed to find the tightest possible bound for a given \(\alpha\).

This distinction has significant practical implications. In a bandit setting, optimizing a fixed ERM can often be reduced to a standard mean-based problem by applying an exponential utility transformation to the rewards. Addressing the EVaR optimization problem, however, presents a greater technical challenge due to the nested infimum over \(z\), which is the novel direction explored in this work.

Our approach to the Best Arm Identification for EVaR (BAI-EVaR) is inspired by the BAI-CVaR framework of \citet{agrawal2021BAIcvar}. We leverage both the MGF and the distributionally robust representations of EVaR. This dual perspective is crucial as it facilitates a clean analysis of two key KL-projection terms that arise in our derivations.

%%%%%%%%%
\section{Assumptions and Preliminaries }
\label{sec:assum}

We start with a simple remark: if the set of distribution considered is limited to a one–parameter canonical exponential family, the BAI-EVaR problem is actually already solved by~\cite{garivier2016optimal}.
\begin{proposition}
\label{prp:evar_spef}
Let $\{P_\theta\}_{\theta\in(a,b)}$ be a one–parameter canonical exponential family
with log-partition $\psi$ (strictly convex). For $\alpha\in(0,1)$, write
\[
\mathrm{EVaR}_\alpha(P_\theta)
\;=\;
\inf_{z>0}\ \frac{\psi(\theta+z)-\psi(\theta)+\rho}{z}.
\]
Then $\theta\mapsto \mathrm{EVaR}_\alpha(P_\theta)$ is \emph{strictly increasing}.
Consequently, since $\mu(\theta)=\psi'(\theta)$ is strictly increasing, the arm minimizing
$\mathrm{EVaR}_\alpha$ coincides with the arm minimizing the mean.
\end{proposition}

\begin{proof}
For each fixed $z>0$, the map
$g_\theta(z):=\big(\psi(\theta+z)-\psi(\theta)\big)/z$
satisfies
$\partial_\theta g_\theta(z)=\big(\psi'(\theta+z)-\psi'(\theta)\big)/z>0$
by strict convexity of $\psi$.
Hence $g_\theta(z)+\tfrac{1}{z}\rho$ is strictly increasing in $\theta$ for every $z>0$.
Taking the infimum over $z>0$ (pointwise infimum of strictly increasing functions)
gives that $\mathrm{EVaR}_\alpha(P_\theta)$ is increasing; on $[0,1]$ the infimum is attained,
so the increase is strict.
\end{proof}

So as to focus on the specificities of EVaR and not of tail properties of the distributions, we focus in the sequel on the set of probability distributions supported on the interval, say $[0,1]$, and denote it by $\Lc := \{\kappa \in \Pc(\R) : \operatorname{supp}(\kappa)\subseteq[0,1]\}$. It is is known to be a compact space under the topology of weak convergence.

\fix{This section reviews essential properties of EVaR and establishes the mathematical tools for our analysis. Specifically, we establish the continuity of our risk measure (Lemma~\ref{lem:evar-cont}) and guarantee the compactness of its level sets on $\Pc([0,1])$ under weak convergence (Corollary~\ref{cor:evar-level-sets}). This latter property is critical, as it enables a direct application of Berge’s Maximum Theorem to our EVaR-BAI optimization problems. This application yields the existence of optimizers, the continuity of value functions, and the upper hemicontinuity of the argmin/argmax correspondences used throughout our framework.
}

\begin{remark}\label{rem:EVaR_conc}
$\EVaR_\alpha(\kappa)$ is concave in the distribution $\kappa$ and convex in the random variable $X$.
\end{remark}
\begin{lemma}\label{lem:evar-cont}
Fix $\alpha\in(0,1)$ and set $\rho:=-\log(1-\alpha)$. For $X\in[0,1]$ and any sequence
$\kappa_n\Rightarrow\kappa$ (weakly) in $\Pc([0,1])$,
\[
\EVaR_\alpha(\kappa_n)
=\inf_{z>0}\frac{\log\E_{\kappa_n}[e^{z X}]+\rho}{z}
\ \longrightarrow\
\EVaR_\alpha(\kappa)=\inf_{z>0}\frac{\log\E_{\kappa}[e^{z X}]+\rho}{z}.
\]
\end{lemma}

See the proof of the lemma in Appendix~\ref{prf:lem:evar-cont}. %\ref{sec:prf:evar_cont}.%or\ref{app:prf:assum:evar_cont} or ~\ref{app:prf:lem_evar_cont} ?
 Continuity of \(\EVaR_\alpha\) on \(\Pc([0,1])\) immediately gives compact level sets because \(\Pc([0,1])\) itself is compact in the weak topology.
\begin{corollary}\label{cor:evar-level-sets}
Fix $\alpha\in(0,1)$ and consider $\EVaR_\alpha:\Pc([0,1])\to\R$.
Then for every $c\in\R$ the \emph{sublevel} and \emph{superlevel} sets
\[
\{\kappa\in\Pc([0,1]):\ \EVaR_\alpha(\kappa)\le c\}
\quad\text{and}\quad
\{\kappa\in\Pc([0,1]):\ \EVaR_\alpha(\kappa)\ge c\}
\]
are closed in the weak topology. Consequently, they are compact since $\Pc([0,1])$ is compact.
%Moreover, the equality set $\{\kappa:\EVaR_\alpha(\kappa)=c\}$ is closed (hence compact) as well.
\end{corollary}

\begin{proof}
By Lemma~\ref{lem:evar-cont}, $\EVaR_\alpha$ is continuous on $\Pc([0,1])$ (weak topology).
Therefore the sets above are preimages of the closed sets $(-\infty,c]$, $[c,\infty)$, and $\{c\}$
under a continuous map, hence are closed. Compactness follows from the compactness of $\Pc([0,1])$. See Lemma~\ref{lem:detailed-levelset} for more details.
\end{proof}

\begin{lemma} \label{lem:rangeEVaR}
Fix $\alpha\in(0,1)$ and let $e_\alpha(\eta)=\inf_{z>0}\tfrac{1}{z}\big(\log\mathbb E_\eta[e^{z X}]-\log(1-\alpha)\big)$
denote $\mathrm{EVaR}_\alpha$ of $X\sim\eta$. If $\eta\in\mathcal P([0,1])$, then
\[
e_\alpha(\eta)\in[0,1].
\]
Moreover, the extremes are attained: $e_\alpha(\delta_0)=0$ and $e_\alpha(\delta_1)=1$. Consequently, the natural
EVaR range set is \(
D := [0,1]\) and  \(D^\circ=(0,1)\) denotes its interior.
\end{lemma}
\begin{definition}[Lower semicontinuity]
A function $f:(\mathcal X,\tau)\to(-\infty,+\infty]$ is \emph{lower semicontinuous (l.s.c.) at $x\in\mathcal X$} if \(
\liminf_{y\to x} f(y)\ \ge\ f(x)\).
Equivalently (sequential form in metric spaces), $f$ is lower semicontinous \ iff for every sequence $x_n\to x$, \(f(x)\ \le\ \liminf_{n\to\infty} f(x_n)\).
\end{definition}
\begin{remark}
%[Compactness $\Pc([0,1])$ $\Rightarrow$ attainment]
\label{rem:compact-attain}
The space $\Pc([0,1])$ of Borel probability measures on $[0,1]$ is compact under weak convergence.
Hence, whenever the objective functional is l.s.c.
and the feasible set is closed in the weak topology,
the minimum is \emph{attained}.
We use the compactness to guarantee existence of minimizers combining with l.s.c. objectives and closed
feasible sets.
\end{remark}

\begin{lemma}\label{lem:lsc}
Fix \(\eta\in \Pc([0,1])\). The map \(\kappa\mapsto \KL(\eta\Vert \kappa)\) is l.s.c on \(\Pc([0,1])\) for the weak topology. In particular, if \(\eta\not\ll \kappa\), then \(\KL(\eta\Vert \kappa)=+\infty\).
\end{lemma}

%\noindent Any problem of the form\(\min_{\kappa\in K} \KL(\eta\Vert \kappa)\) over a nonempty closed \(K\subset \Pc([0,1])\) admits a minimizer.

% \begin{proposition}[Weierstrass-type attainment]\label{prop:attainment}
% Let $K\subset\mathcal X$ be compact and $f:K\to(-\infty,+\infty]$ be l.s.c. If $F\subset K$ is closed and nonempty, then
% $\inf_{x\in F} f(x)$ is attained: there exists $x^\star\in F$ with $f(x^\star)=\inf_{x\in F} f(x)$.
% \end{proposition}
%%%%%%%%%
\section{Model and the Lower Bound}
\label{sec:lowerbound}

\paragraph{Problem setup (BAI–EVaR).}
Fix an integer $K \ge 2$. Let $\mathcal P([0,1])$ denote the set of probability
measures supported on $[0,1]$, and let
$\mathcal M \;:=\; \mathcal P([0,1])^K$
be the set of $K$-arm bandit instances. A problem instance is
$\nu=(\nu_1,\ldots,\nu_K)\in\mathcal M$, where arm $i$ yields i.i.d.\ draws
$X_i\sim \nu_i$. For a risk level $\alpha\in(0,1)$, EVaR is defined in \eqref{def:evar:mgf} or \eqref{def:evar:KL} of a
distribution $\eta\in\mathcal P([0,1])$.
% is
% \[
% \EVaR_\alpha(\eta)
% \;:=\; \inf_{t>0}\; \frac{1}{t}\Big(\log \E_{X\sim\eta}[e^{tX}] - \log(1-\alpha)\Big).
% \]
The best arm under EVaR is
$i^\star(\nu) \;:=\; \arg\min_{i\in\{1,\dots,K\}} \EVaR_\alpha(\nu_i)$.
A BAI–EVaR algorithm is a sampling rule, a stopping time
$\tau$, and a recommendation $\hat{\imath}_\tau$. For a confidence level
$\delta\in(0,1)$, the algorithm is \emph{$\delta$-correct} if for all
$\nu\in\mathcal M$, the algorithm declares the best arm with probability higher tha $1-\delta$,
and its sample complexity is \(\E_\nu[\tau_\delta]\), which we aim to minimize.

The general lower bound as in Garivier-Kaufmann \cite{garivier2016optimal} remains the same in the general form
\begin{equation}\label{th:GK2016}
\mathbb{E}(\tau_{\delta}) \ge T(\mu)^{-1} \log\frac{1}{4\delta}
\quad 
\text{where} \quad T(\mu) = \sup_{t\in\Sigma_{K}} \inf_{\nu\in\mathcal{A}_{1}^{c}}\sum_{a=1}^{K}t_{a}\KL(\mu_{a},\nu_{a}), \quad \text{and } \mathcal{A}_{j}^{c}=\mathcal{M}\backslash\mathcal{A}_{j}. 
\end{equation}

% We consider $\delta$-correct algorithms for identifying the arm with minimum EVaR, acting on bandit problems in $\mathcal{M}$. While ensuring $\delta$-correctness, the aim is to minimize the sample complexity, i.e., expected number of samples generated by the algorithm before it terminates.

{
%Let $\mu \in \Lc$ denote the given bandit problem and 
For the ease of notation, we assume without loss of generality that the best-EVaR arm in $\mu$ is arm index $1$. 
Let $\Sigma_{K}$ denote the probability simplex in $\mathbb{R}^{K}$, $\mathcal{A}_{j}$ denote the collection of all bandit problems in $\mathcal{M}$ which have arm $j$ as the best-EVaR arm, $\tau_{\delta}$ be the stopping time for the $\delta$-correct algorithm, $N_{a}(\tau)$ denote the number of times arm $a$ has been sampled by the algorithm.
%and for a set $S$, let $S^{o}$ denote its interior.
%[Similar to Lemma 3.1 from \cite{agrawal2021BAIcvar}]
To understand the lower bound and later propose a $\delta$-correct
 algorithm, we need to have a better understanding of $T(\mu)$.

 Analogous to the general heavy tail distribution \cite{agrawalBAImean}
and CVaR constraint \cite{agrawal2021BAIcvar}, one can simplify $T(\mu)$ in \eqref{th:GK2016} as
% $\begin{lemma}
% For $\mu \in \mathcal{A}_{1}$, the inner minimization problem in $T(\mu)$ equals
% \[
% \min_{j \neq 1} \inf_{x \le y} \{t_{1} \KLiU
% (\mu_{1},y) + t_{j}\KLiL(\mu_{j},x)\},
% \]
% and hence
\begin{equation}
T(\mu) = \sup_{t \in \Sigma_K} \min_{j \neq 1} \inf_{x \le y} \{t_{1} \KLiU
(\mu_{1},y) + t_j \KLiL(\mu_j, x)\}.
\label{Tmu}
\end{equation}
%\end{lemma}
  where the high-level structure and formulas of the lower bound and $T(\mu)$ remain the same, but the underlying properties of the $\KL_{\inf}$ functionals are different and important to discuss carefully and  necessary for understanding of lower bound and our proposed algorithm.

\begin{lemma}\label{lem:KLiLKLiU-cont}
On $\Pc([0,1])\times[0,1]$ (weak topology on $\Pc([0,1])$), the maps
\[
(\eta,x)\longmapsto \KLiL(\eta,x)
\quad\text{and}\quad
(\eta,x)\longmapsto \KLiU(\eta,x)
\]
are continuous.
\end{lemma}
\fix{The proof is deferred to Appendix~\ref{app:prf:lem:KLiLKLiU-cont}, as it relies on details of the KL-projection that are presented first.}
\begin{lemma}
\label{lem:stab_EVaR}
For $\nu\in \mathcal P([0,1])^K$, let $t^*(\nu)\subseteq\Sigma_K$ denote the set of maximizers of \eqref{Tmu}.
Then $t^*(\nu)$ is nonempty, compact, convex, and the correspondence $\nu\mapsto t^*(\nu)$ is upper hemicontinuous.
\end{lemma}
Proof is given in Appendix~\ref{prf:lem:stab_EVaR}.
We need this lemma to ensure stability of the oracle weights, upper hemicontinuity (and convexity) of $t^*(\nu)$ guarantees that as the empirical instance $\hat\nu_n$ approaches the true $\nu$, any computed maximizer stays close to an optimal one and we use it in the analysis to keep the C-tracking  \cite{garivier2016optimal} allocations near-optimal along the run and to establish the $\delta$-correctness and asymptotic sample-complexity bound via continuity of the value $T(\nu)$ and the GLRT lower bound.

%  {\color{red} I talk here on KLU and KLL, in a lemma claim their convex/cavity, then a theorem for each on their dual representations}
}
\subsection{KL functionals and Dual representations}
% \documentclass[11pt]{article}

% % ---------- Packages ----------
% \usepackage[a4paper,margin=1in]{geometry}
% \usepackage{amsmath,amssymb,amsthm,mathtools}
% \usepackage{microtype}
% \usepackage{bm}
% \usepackage{bbm}
% \usepackage{xcolor}
% \usepackage{hyperref}
% \hypersetup{colorlinks=true,linkcolor=blue!50!black,citecolor=blue!50!black,urlcolor=blue!50!black}

% % ---------- Theorem styles ----------
% \newtheorem{definition}{Definition}
% \newtheorem{lemma}{Lemma}
% \newtheorem{proposition}{Proposition}
% \newtheorem{remark}{Remark}

% % ---------- Macros ----------
% \newcommand{\KL}{\mathrm{KL}}
% \newcommand{\E}{\mathbb{E}}
% \newcommand{\Pcal}{\mathcal{P}}
% \newcommand{\Mcal}{\mathcal{M}_+}
% \newcommand{\dd}{\,\mathrm{d}}
% \newcommand{\1}{\mathbbm{1}}

% \begin{document}

% \title{EVaR--$\mathrm{KLiU}_{\inf}$ in the Style of the CVaR Paper (Appendix D.1)}
% \author{}
% \date{}
% \maketitle

%\section{Setup and Problem Statement}

% Fix a tail level $\pi\in(0,1)$ and write
% \[
% \rho \;=\; \log\!\frac{1}{1-\pi}.
% \]
% We work on support $[0,1]$ so no additional moment constraint is needed. Let $\eta\in \Pcal([0,1])$ be a reference distribution (e.g.\ an empirical law).

% \begin{definition}[EVaR robust form]
% For $\kappa\in \Pcal([0,1])$,
% \[
% \mathrm{EVaR}_\pi(\kappa)
% \;=\;
% \sup_{Q \ll \kappa \,:\, \KL(Q\Vert \kappa)\le \rho}\ \E_Q[X].
% \]
% \end{definition}

\label{sec:dis_str_kl}
%\section{KL functionals and Dual representations}
\textbf{Discussion on non/convexity of each KL:}
Recall that \(\KL(\cdot\Vert\cdot)\) is strictly convex in its second argument. EVaR is concave in distribution, by Remark~\eqref{rem:EVaR_conc}. We now discuss feasible regions for each KL-projection.
In \(\KLiU\), the constraint \(\EVaR(\cdot)\ge \nu\) is a superlevel set of the EVaR functional, thus the set is convex; therefore \(\KLiU\) defines a convex program. 
In contrast, \(\KLiL\) imposes \(\EVaR(\cdot)\le \nu\), a sublevel set of a concave functional, which is generally nonconvex. Hence \(\KLiL\) is a nonconvex problem.

%\subsection{Structural Properties of \(\KLiU\)}

%In the $\KLiU$ problem we impose the lower constraint $\EVaR_{\alpha}(\kappa)\ge \nu$. 
\noindent
\paragraph{\textbf{$\KLiU$ projection:}}
Under the MGF-based representation of EVaR in \eqref{def:evar:mgf}, which takes an infimum over $z>0$, the feasibility constraint must hold \emph{uniformly for all} $z>0$. This renders the associated convex optimization a \emph{semi-infinite} program, which is computationally burdensome. To obtain a tractable formulation, we instead adopt the (equivalent) robust KL-ball definition of EVaR in \eqref{def:evar:KL}, which introduces an auxiliary probability measure $Q$ and leads to the following primal problem:

% \begin{equation}
% \begin{aligned}
% \min_{\substack{\kappa \in \Pcal([0,1]): 
% \\ Q \in \Pcal([0,1])}}
% \quad & \KL(\eta\Vert \kappa)
% \\
% \text{s.t.}\quad & \E_Q[X]\ \ge\ \nu,
% \\
% & \KL(Q\Vert \kappa)\ \le\ \rho.
% \end{aligned}
% \end{equation}

\begin{equation}
\label{eq:EVaR-KLiU-primal}
\begin{aligned}
\KLiU(\eta, \nu) = \min_{\kappa,\,Q\in \Pcal([0,1])}\quad & \KL(\eta\Vert \kappa)\\
\text{s.t.}\quad & \E_Q[X]\ \ge\ \nu,\qquad \KL(Q\Vert \kappa)\ \le\ \rho.
\end{aligned}
\end{equation}
%\noindent
%\textit{Note.} The absolute continuity $Q\ll\kappa$ is implied whenever $\KL(Q\Vert\kappa)<\infty$, so it need not be written explicitly.

%\begin{proposition}\label{prp:KLiUdual}
%    The dual representation of \eqref{eq:EVaR-KLiU-primal} simplifies to
%     \begin{equation}
%         \max_{\lambda_1,\lambda_3 \in \mathcal G} \E_{\eta}\Bigg[ \log\bigg(1+\lambda_3(1-\exp(\frac{\lambda_1}{\lambda_3}(X-\nu)+\rho))\bigg)\Bigg]
%     \end{equation}
%  where we define  $\forall x\in\supp(\eta)$
% \[
% \mathcal G(\eta;\nu,\rho)
% :=\Bigl\{(\lambda_1,\lambda_3)\in[0,\infty)^2:\ 
% \lambda_3>0,\ 
% 1+\lambda_3\bigl(1-\exp\bigl(\tfrac{\lambda_1}{\lambda_3}(x-\nu)+\rho\bigr)\bigr)>0
% \ \
% \Bigr\}
% \ \cup\ \{(\lambda_1,\lambda_3):\lambda_3=0\}.
% \]

% The union with \(\{\lambda_3=0\}\) covers the boundary case (for which the dual
% integrand equals \(\log 1=0\)). For \(\lambda_3>0\), the constraint can be written equivalently as
% \[
% \sup_{x\in\supp(\eta)}\ \lambda_3\bigl(\exp(\tfrac{\lambda_1}{\lambda_3}(x-\nu)+\rho)-1\bigr)\ <\ 1.
% \]

% For \(X\in[0,1]\), then for \(a:=\lambda_1/\lambda_3\ge 0\) the function
% \(x\mapsto\exp(a(x-\nu))\) is nondecreasing, so the worst case is at
% \(x_{\max}:=\sup\supp(\eta)\) (typically \(x_{\max}=1\)). Thus
% \[
% (\lambda_1,\lambda_3)\in\mathcal G(\eta;\nu,\rho)
% \iff
% \begin{cases}
% \lambda_3=0, & \text{or}\\[4pt]
% 0<\lambda_3<\dfrac{1}{\,e^{\,\rho+a(x_{\max}-\nu)}-1\,}, & \text{with }a=\lambda_1/\lambda_3.
% \end{cases}
% \]
%\end{proposition}

\begin{theorem}\label{Th:KLiU}
Fix $\rho>0$ and $v\in[0,1]$. For $\eta\in\mathcal P([0,1])$, the dual representation of \eqref{eq:EVaR-KLiU-primal} 
   \begin{equation}
     \KLiU(\eta, \nu) =   \max_{\lambda_1,\lambda_3 \in \mathcal G} \E_{\eta}\Bigg[ \log\bigg(1+\lambda_3(1-\exp(\frac{\lambda_1}{\lambda_3}(X-\nu)+\rho))\bigg)\Bigg]
    \end{equation}
    with the implicit feasibility condition
$
1+\lambda_3\big(1-e^{\frac{\lambda_1}{\lambda_3}(x-v)+\rho}\big)>0
$ for all $x\in\supp(\eta)$.
At any maximizer with $\lambda_3>0$, the primal optimizer $\kappa^*$ satisfies
    \begin{equation}
    \frac{d\kappa^*}{d\eta} = \frac{1}{1+\lambda_3\Bigg(1-\exp\big(\frac{\lambda_1}{\lambda_3}(X-\nu)+\rho\big)\Bigg)},
    \qquad
\frac{dQ^*}{d\kappa^*}(x)=\frac{e^{\frac{\lambda_1}{\lambda_3}x}}{\E_{\kappa^*}[e^{\frac{\lambda_1}{\lambda_3}X}]},
    \end{equation}
and complementary slackness yields 
$\E_{Q^*}[X]=v$ and $\KL(Q^*\Vert\kappa^*)=\rho$, hence 
$\E_{\kappa^*}[e^{\frac{\lambda_1}{\lambda_3}X}]=e^{\frac{\lambda_1}{\lambda_3}v-\rho}$.
The boundary case $\lambda_3=0$ corresponds to $Q^*=\kappa^*$ and yields value $0$ when $\E_{\kappa^*}[X]\ge v$.
\end{theorem}
Complete proof is given in Subsection~\ref{proof:prp:kliU}.

%\subsection{Structural Properties of \(\KLiL\)}

\noindent
\paragraph{\textbf{$\KLiL$ projection:}}
As discussed above, $\KLiL$ is intrinsically nonconvex.
Using the MGF representation of EVaR \eqref{def:evar:mgf}, the sublevel constraint is equivalent to the \emph{existence} of a positive scalar:
\[
\EVaR_{\alpha}(\kappa)\le \nu
\quad\Longleftrightarrow\quad
\exists\,z>0:\ \E_{\kappa}\!\big[e^{z X}\big]\ \le\ e^{-\rho+z\nu},
\qquad \rho:=-\log (1-\alpha).
\]
This yields a one-dimensional outer infimum over $z$ of a convex inner projection problem:
\begin{equation}
\label{eq:EVaR-KLiL-primal-beta}
\KLiL(\eta,\nu)
\;=\;
\inf_{z>0}\;
\min_{\kappa\in \mathcal P([0,1])}
\Big\{\, \KL(\eta\Vert \kappa)\ :\ \E_{\kappa}\!\big[e^{z X}\big]\ \le\ e^{-\rho+z\nu}\,\Big\}.
\end{equation}
Although $\KLiL$ is globally nonconvex, it remains computationally tractable in practice. The outer infimum is a one-dimensional search over $z>0$, and for each fixed $z$ the inner convex projection (or its scalar dual maximization) is much simpler than the $\KLiU$ case.
For each fixed $z$, the inner problem is convex in $\kappa$ (since $\kappa\mapsto \KL(\eta\Vert\kappa)$ is convex and the moment constraint is linear in $\kappa$); its Lagrange dual introduces a single scalar multiplier and takes the familiar log-partition form (cf. Prop.~\ref{Th:KLiL}).

% {\color{red} To Do:}{\color{blue} a discussion on the existence of $z$ that satisfies the condition}

\begin{theorem}\label{Th:KLiL}
Fix $\rho>0$, $\nu\in[0,1]$, and $\eta\in\mathcal P([0,1])$. Then
\[
\KLiL (\eta,\nu)
=\inf_{\ z>0}\ \sup_{\lambda\in\mathcal D(\ z,\nu)}
\ \E_\eta\!\Big[\log\big(1-\lambda\big(e^{-\rho+\ z\nu}-e^{\ z X}\big)\big)\Big],
\]
where
\[
\mathcal D(\ z,\nu)=
\begin{cases}
[0,\infty), & e^{-\rho+\ z\nu}\le 1,\\[4pt]
\big[0,\ 1/(e^{-\rho+\ z\nu}-1)\big), & e^{-\rho+\ z\nu}>1.
\end{cases}
\]
Moreover, for any maximizer $(\ z,\lambda)$ with $e^{-\rho+\ z\nu}>1$ and $\lambda>0$, the primal minimizer is
\[
\frac{d\kappa^*}{d\eta}(x)=\frac{1}{\,1-\lambda\big(e^{-\rho+\ z\nu}-e^{\ z x}\big)}\,,
\]
and the moment constraint is tight:
$\int e^{\ z x}\,d\kappa^*(x)=e^{-\rho+\ z\nu}$.
\end{theorem}

We provide the proof  in Subsection~\ref{sec:prf:KLiL}

\section{Algorithms}
\label{sec:alg:evar}

Fix a bandit instance $\mu=(\mu_a)_{a\in[K]} \in \mathcal{M}$ with rewards supported on $[0,1]$ and EVaR level $\alpha\in(0,1)$. 
Let $\rho:= -\log(1-\alpha)$ and write $\EVaR_\alpha(\eta)$ for the EVaR of a distribution $\eta$. 
Denote by $\hat\mu(n)=(\hat\mu_a(n))_{a\in[K]}$ the vector of empirical measures after $n$ pulls, and by $N_a(n)$ the number of samples drawn from arm $a$ up to time $n$.
Given a bandit problem $\mu\in\Mc$, we propose a sampling rule, a stopping rule, and a recommendation rule.
For sampling rule we use C-Tracking \cite{garivier2016optimal}, we use generalized likelihood ratio test (GLRT) for stopping time and then the algorithm outputs the arm with the minimum EVaR of the corresponding empirical distribution, i.e., if $\tau_{\delta}$ is the stopping time of the algorithm, then it outputs $\argmin_a \EVaR(\hat{\mu}_{\tau_\delta })$. 
The main technical change is instead of classical KL-divergence we have $\KLiU$ and $\KLiL$ which need to be solve at each round with empirical distribution we estimated until then.

\paragraph{Oracle proportions}
Recall the term $T(\mu)$ in the lower bound \ref{Tmu},
% \[
% T(\mu)\;=\;\sup_{t\in\Sigma_K}\;\min_{j\neq 1}\;\inf_{x\le y}
% \Big\{\,t_1\,\KLiU(\mu_1,y)\;+\;t_j\,\KLiL(\mu_j,x)\,\Big\},
% \]
we simplify because the minimum achieves at a common point for the two terms
\[
T(\mu)\;=\;\sup_{t\in\Sigma_K}\;\min_{j\neq 1}\;\inf_{x\in [\EVaR(\nu_1),\EVaR(\nu_j)]}
\Big\{\,t_1\,\KLiU(\mu_1,x)\;+\;t_j\,\KLiL(\mu_j,x)\,\Big\},
\]

\paragraph{Sampling rule (C-tracking of the oracle proportions).}
At round $n$, we have  empirical distributions 
%($\Pi$ is a stable projection, e.g., in Kolmogorov distance):\[\hat\mu(n)\;:=\;\Pi\!\big(\widehat\mu(n)\big)\in\mathcal{M}.\]
and compute a maximizer
\[
t^*\big(\hat\mu(n)\big)\;\in\;\argmax_{t\in\Sigma_K}\;\min_{j\neq 1}\ \inf_{x}\Big\{t_1\,\KLiU\!\big(\hat\mu_1(n),x\big)+t_j\,\KLiL\!\big(\hat\mu_j(n),x\big)\Big\},
\]
and pull the arm prescribed by the C-tracking rule that tracks $t^*\!\big(\hat\mu(n)\big)$ with mild forced exploration \cite{garivier2016optimal}. % C-tracking mechanics as in Track-and-Stop

\paragraph{Stopping rule}
Let $i(n)\in\argmin_{a} \EVaR_\alpha\!\big(\widehat\mu_a(n)\big)$ be the empirically best-EVaR arm. 
Define the generalized likelihood ratio statistic
\[
Z_{i(n)}(n)\;:=\;\min_{a\neq i(n)}\ \inf_{x\le y}\Big\{\,N_{i(n)}(n)\,\KLiU\!\big(\widehat\mu_{i(n)}(n),y\big)\;+\;N_a(n)\,\KLiL\!\big(\widehat\mu_a(n),x\big)\,\Big\},
\]
and stop at the hitting time
\[
\tau_\delta\;:=\;\inf\Big\{n\ge 1:\ Z_{i(n)}(n)\ \ge\ \beta(n,\delta)\Big\},
\qquad
\beta(n,\delta)\;=\;\log\!\Big(\tfrac{K-1}{\delta}\Big)+3\log(n+1)+2.
\]
At time $\tau_\delta$, output the arm minimizing the empirical EVaR 
$\widehat a_\delta\;\in\;\argmin_{a\in[K]}\ \EVaR_\alpha\!\big(\widehat\mu_a(\tau_\delta)\big)$.

 The algorithm makes an error if at the stopping time, there is an arm $j\neq 1$ such that its EVaR is not the lowest.
Let the error event be denoted by $\Ec$.
\begin{theorem}\label{Th:upp_sample}
For $\delta\in (0,1)$ and $\mu \in \Mc$ the proposed algorithm sis $\delta$-correct algorithm with chose $\beta(\delta, t)$, satisfies
\[
\limsup_{\delta \to 0} \frac{\E_{\mu}[\tau_\delta]}{\log(\frac{1}{\delta})} \leq \frac{1}{T(\mu)}
\]
\end{theorem}
Proof of Theorem~\ref{Th:upp_sample} is given in Appendix~\ref{app:subsec:prf:Th:upp_sample}.

By writing both KL functionals through the EVaR duals and
constructing mixture-supermartingale argument we deliver a time-uniform deviation inequality that leads to the threshold $\beta(n,\delta)$ above. 
With C-tracking, the sampling proportions converge to the oracle set $t^*(\mu)$, yielding $\delta$-correctness and asymptotic optimality.
The key technical change compare to \cite{agrawal2021BAIcvar} is that we evaluate Evar-tailored $\KLiU$ and $\KLiL$ (discussed in Section~\ref{sec:dis_str_kl}) at each round on empirical inputs.  

%%%%%%%%%
\cmt{\section{Experiment and Discussions???}
\label{sec:experiment}
}
%%%%%%%%%
\section{Conclusion}
We provide the first asymptotically optimal algorithms for best-arm identification under Entropic Value at Risk (EVaR). By exploiting EVaR’s entropic dual form, we obtain non-parametric procedures with provable guarantees. The analysis departs from CVaR/VaR approaches \cite{agrawal2021BAIcvar}, tailored to entropic risk measures.

Methodologically, our contribution is twofold. We derive \emph{anytime EVaR-KL concentration} via mixture supermartingales constructed by using dual representations, which yields tight one-sided information projections under EVaR constraints. This delivers $\delta$-correct stopping rules and instance-dependent sample complexity characterized by a risk-aware \emph{characteristic time}.  We instantiate a GLRT/Track\&Stop-style procedure whose asymptotic optimality matches the information-theoretic lower bound in the EVaR setting, without parametric assumptions and for general laws supported on $[0,1]$.

\fix{Two promising directions emerge from our results and prior work on BAI with risk constraints: developing a unified framework that treats coherent risks under a single universal constraint class (via common duality and KL-projection primitives). Such a framework could yield instance-optimal lower bounds and algorithmic templates that specialize automatically to each risk, clarifying when the same sampling and stopping rules remain optimal across the entire family of risk measures.}

% Acknowledgments--Will not appear in anonymized version
%\acks{We thank our colleagues and funding agencies;
 % \texttt{\textbackslash documentclass[anon]\{alt2026\}}
%automatically hides this text.}

\bibliography{main.bib}

\appendix
% \crefalias{section}{appendix} % uncomment if you are using cleveref
%\section{Deferred proofs}
%We present here the proof of the main lemma and theorems.
\label{Appendix_new}

\section{Proofs related to the Continuity of EVaR}
\fix{
This appendix provides proofs for the technical results supporting Section~\ref{sec:assum}. First, we present a continuity lemma for the log-moment generating function. Then, we investigate the behavior of EVaR, specifically when its minimizer approaches $\infty$. This lemma and its corollary are essential for the proof of the joint continuity of the KL-projection (Lemma~\ref{lem:EVaR-KLL-usc}). Finally, we establish the continuity of EVaR in distribution (Lemma~\ref{lem:evar-cont}) and provide the detailed proof for Corollary~\ref{cor:evar-level-sets}.
}
\begin{lemma}[Continuity of Log-MGF]\label{lem:logmgf-cont}
Let $X\in[0,1]$ and, for $\kappa\in\Pc([0,1])$ and $z>0$, define
$\Lambda(\kappa, z)\;:=\;\log \E_\kappa\!\big[e^{z X}\big]$. Then $(\kappa,z)\mapsto \Lambda(\kappa, z)$ is continuous on $\Pc([0,1])\times(0,\infty)$
with the weak topology on $\Pc([0,1])$.
\end{lemma}
%Proof is given Appendix~\ref{app:prf:lem:logmgf-cont}.
\begin{proof}\label{app:prf:lem:logmgf-cont}
Fix $(\kappa_0,\theta_0)\in\Pc([0,1])\times(0,\infty)$ and let $(\kappa_n,\theta_n)\to(\kappa_0,\theta_0)$.
Because $x\mapsto e^{\theta x}$ is bounded and continuous on $[0,1]$ for each fixed $\theta>0$,
Portmanteau’s theorem gives $\E_{\kappa_n}[e^{\theta_0 X}]\to\E_{\kappa_0}[e^{\theta_0 X}]$.

For joint continuity in $(\kappa,\theta)$, note that on any compact $\Theta\subset(0,\infty)$ we have the uniform bound
$e^{\theta x}\le e^{\sup\Theta}$ for all $(\theta,x)\in\Theta\times[0,1]$ and the map
$(\theta,x)\mapsto e^{\theta x}$ is continuous. Hence, given $(\kappa_n,\theta_n)\to(\kappa_0,\theta_0)$ with all $\theta_n\in\Theta$,
\[
\E_{\kappa_n}[e^{\theta_n X}]-\E_{\kappa_0}[e^{\theta_0 X}]
=\underbrace{\big(\E_{\kappa_n}[e^{\theta_n X}]-\E_{\kappa_n}[e^{\theta_0 X}]\big)}_{\to 0\ \text{by dominated conv.}}
+\underbrace{\big(\E_{\kappa_n}[e^{\theta_0 X}]-\E_{\kappa_0}[e^{\theta_0 X}]\big)}_{\to 0\ \text{by Portmanteau}} \ \longrightarrow\ 0.
\]
Since $\E_\kappa[e^{\theta X}]\ge 1$ for all $\theta>0$, the logarithm is continuous at these values, yielding
$\Lambda(\kappa_n,\theta_n)\to \Lambda(\kappa_0,\theta_0)$.
\end{proof}
\begin{lemma}\label{lem:f-bounds}
Let $X\in[0,1]$, $\eta\in\Pc([0,1])$, $\rho:=-\log(1-\alpha)>0$, and
\[
f_\eta(z):=\frac{\log \E_\eta[e^{zX}]+\rho}{z},\qquad \xmax(\eta)\ :=\ \operatorname*{ess\,sup}_{\eta} X \in[0,1].
\]
Then for every $z>0$,
\[
\xmax(\eta)\;+\;\frac{\log \eta\{X=\xmax(\eta)\}+\rho}{z}
\;\le\;
f_\eta(z)
\;\le\;
\xmax(\eta)\;+\;\frac{\rho}{z},
\]
(with the convention $\log 0=-\infty$), and
\[
\lim_{z\to\infty} f_\eta(z)=\xmax(\eta).
\]
\end{lemma}
\begin{proof}
\label{prf:lem:f-bounds}
Since $X\le \xmax(\eta)$ a.s., $\E_\eta[e^{zX}]\le e^{z \xmax(\eta)}$, hence
$f_\eta(z)\le \xmax(\eta)+\rho/z$. On the other hand, $\E_\eta[e^{zX}]\ge \eta\{X=\xmax(\eta)\}\,e^{z \cdot\xmax(\eta)}$, so
$f_\eta(z)\ge \xmax(\eta)+(\log \eta\{X=\xmax(\eta)\}+\rho)/z$. The limit follows by squeezing.
\end{proof}

\begin{corollary}
\label{cor:no-equality}
If $\rho>0$ (i.e., $\alpha\in(0,1)$), then for any finite $z>0$,
\(f_\eta(z)=\xmax(\eta)\quad\text{never holds.}
\)
Indeed, the upper bound in Lemma~\ref{lem:f-bounds} gives $f_\eta(z)\le \xmax(\eta)+\rho/z$ with $\rho/z>0$,
while the lower bound can be $<\xmax(\eta)$ if $\eta\{X=\xmax(\eta)\}<1$. Only the \emph{limit} satisfies
$\lim_{z\to\infty} f_\eta(z)=\xmax(\eta)$.
In 
\end{corollary}
This is important to characterizing the EVaR minimizer. Let us set $G_\eta(z):=z\,\frac{d}{dz}\log \E_\eta[e^{zX}]-\log \E_\eta[e^{zX}]$. A direct differentiation gives
\[
f_\eta'(z)\;=\;\frac{G_\eta(z)-\rho}{z^2},\qquad
G_\eta'(z)\;=\; z\frac{\mathrm d^2}{\mathrm d z^2} (\log \E_\eta[e^{zX}] )=\frac{\mathbb E_\eta[X^2 e^{zX}]}{\mathbb E_\eta[e^{zX}]}
  - \left(\frac{\mathbb E_\eta[X e^{zX}]}{\mathbb E_\eta[e^{zX}]}\right)^2
\]
the expression for $G_\eta'(z)$, the variance of the tilted measure \(Q_z\)  defined by
\(
\dfrac{dQ_z}{d\eta}(x)=\dfrac{e^{z x}}{\mathbb E_\eta[e^{zX}]}\,
\), is nonnegative. Thus, $G_\eta$ is nondecreasing in $z$.
Any finite minimizer $z^*$ of $f_\eta(z)$ must satisfy $f'_\eta(z^*)=0$, i.e. $G_\eta(z^*) = \rho$.

Because $G_\eta$ is nondecreasing with $G_\eta(0)=0$, there are two regimes:
\[
\sup_{z\ge 0}G_\eta(z)\ \begin{cases}
>\ \rho &\Rightarrow\ \exists \ z^*\in(0,\infty)\ \text{s.t. }G_\eta(z^*)=\rho
\ \Rightarrow\ \EVaR_\alpha(\eta)=f_\eta(z^*)\ (<x_{\max}),\\[4pt]
\le\ \rho &\Rightarrow\ G_\eta(z)<\rho, \forall z\ \Rightarrow\ f_\eta'(z)<0, \forall z
\ \Rightarrow\ \EVaR_\alpha(\eta)=\lim_{z\to\infty}f_\eta(z)=x_{\max}(\eta).
\end{cases}
\]
So, the minimizer is either finite or we reach the boundary of EVaR. But that does not effect the continuity of EVaR by Lemma~\ref{lem:f-bounds}.
%We need a technical tools to tackle this possibility of infinite minimizer.
%The following Lemma converts a noncompact infimum (with a possible boundary optimizer at $z = \infty$)
%into a compact minimum with a continuous objective.
\subsection{Proof of Lemma~\ref{lem:evar-cont}}
\begin{proof}
\label{prf:lem:evar-cont}
Fix $\varepsilon>0$ and choose $\hat z>0$ with
$f_\eta(\hat z)\le \EVaR_\alpha(\eta)+\varepsilon$.
Since $x\mapsto e^{\hat z x}$ is bounded and continuous on $[0,1]$,
$\Lambda_{\eta_n}(\hat z)\to \Lambda_\eta(\hat z)$, hence
$f_{\eta_n}(\hat z)\to f_\eta(\hat z)$.
Thus
\[
\limsup_{n\to\infty}\EVaR_\alpha(\eta_n)
\ \le\ \limsup_n f_{\eta_n}(\hat z)
=\ f_\eta(\hat z)
\ \le\ \EVaR_\alpha(\eta)+\varepsilon.
\]
Let $\varepsilon\downarrow 0$. For each $n$ pick $z_n>0$ with
$\EVaR_\alpha(\eta_n)\ge f_{\eta_n}(z_n)-\varepsilon$.
Consider two cases.

\textbf{Case 1: $(z_n)$ is bounded.}
Up to a subsequence $z_n\to z^*>0$.
Because $(x,z)\mapsto e^{z x}$ is uniformly continuous on $[0,1]\times[\underline z,\overline z]$,
$\Lambda_{\eta_n}(z_n)\to \Lambda_\eta(z^*)$, hence
$f_{\eta_n}(z_n)\to f_\eta(z^*)\ge \EVaR_\alpha(\eta)$.
Therefore $\liminf_n \EVaR_\alpha(\eta_n)\ge \EVaR_\alpha(\eta)-\varepsilon$.

\textbf{Case 2: $z_n\to\infty$.}
By Lemma~\ref{lem:f-bounds}, $f_{\eta_n}(z)\uparrow \xmax(\eta_n)$ as $z\to\infty$; thus for large $n$,
$f_{\eta_n}(z_n)\ge \xmax(\eta_n)-\varepsilon$ and
\[
\liminf_{n\to\infty}\EVaR_\alpha(\eta_n)
\ \ge\ \liminf_n\big(\xmax(\eta_n)-2\varepsilon\big)
\ \ge\ \xmax(\eta)-2\varepsilon
\ \ge\ \EVaR_\alpha(\eta)-2\varepsilon,
\]
using $\xmax(\cdot)$ is lower semicontinuous under weak convergence and $\EVaR_\alpha(\eta)\le \xmax(\eta)$.
Let $\varepsilon\downarrow 0$, the proof is complete.
\end{proof}
%{\color{blue} Another approach is to use compactification on noncompact z, then directly use the Berge's theorem.}

Here, we explain exact description of MGF level sets.
\begin{lemma}\label{lem:detailed-levelset}
Fix $z>0$ and define $F:\Pc([0,1])\to[1,e^{z}]$ by $F(\kappa):=\E_\kappa[e^{zX}]$.
Then $\mathrm{range}(F)=[1,e^{z}]$ and, for any $c\in\R$,
\[
\begin{aligned}
&\{\,\kappa:\ F(\kappa)\le c\,\}=
\begin{cases}
\emptyset, & c<1,\\
\Pc([0,1]), & c\ge e^{z},\\
\text{nonempty, proper, closed}, & c\in[1,e^{z}),
\end{cases}
\\[4pt]
&\{\,\kappa:\ F(\kappa)\ge c\,\}=
\begin{cases}
\Pc([0,1]), & c\le 1,\\
\emptyset, & c>e^{z},\\
\text{nonempty, proper, closed}, & c\in(1,e^{z}],
\end{cases}
\\[4pt]
&\{\,\kappa:\ F(\kappa)= c\,\}=
\begin{cases}
\emptyset, & c\notin[1,e^{z}],\\
\{\delta_0\}, & c=1,\\
\{\delta_1\}, & c=e^{z},\\
\text{nonempty, closed}, & c\in(1,e^{z}).
\end{cases}
\end{aligned}
\]
In particular, all three sets are closed (hence compact) for every $c\in\R$.
\end{lemma}

\begin{proof}
Since $X\in[0,1]$, we have $1\le e^{zX}\le e^{z}$ a.s., so $F(\kappa)\in[1,e^{z}]$ for all $\kappa$.
Surjectivity onto $[1,e^{z}]$ follows by mixing $\delta_0$ and $\delta_1$:
for $\lambda\in[0,1]$, let $\kappa_\lambda:=(1-\lambda)\delta_0+\lambda\delta_1$; then
$F(\kappa_\lambda)=1+\lambda(e^{z}-1)$ sweeps $[1,e^{z}]$.
The case disjunctions for the sublevel, superlevel, and level sets are immediate from
$F(\kappa)\in[1,e^{z}]$ and surjectivity.
Closedness of each set follows because $F$ is continuous on $\Pc([0,1])$
(Portmanteau, as $x\mapsto e^{z x}$ is bounded and continuous), hence preimages of closed subsets of $\R$ are closed.
Compactness uses compactness of $\Pc([0,1])$.
\end{proof}

\section{Related Proofs to Section~\ref{sec:lowerbound}}
\label{app:sec:lowerbound}
In this section, we eventually provide the proof of Lemma~\ref{lem:stab_EVaR}. 

\paragraph{Upper-hemicontinuity of $t^\ast$.}
Let $\nu \in A_j \cap \mathcal{M}$, i.e., the best-EVaR arm in $\nu$ is arm $j$. Then from Lemma~\ref{lem:rangeEVaR}, for all $i \in [K]$, $\EVaR(\nu_i) \in D$. Let $t^\ast(\nu)$ be the set of maximizers in
\[
V(\nu) \;=\; \max_{t \in \Sigma_K} \; \min_{a \neq j} \; g_{a,j}(\nu,t),
\]
where
\[
g_{a,j}(\nu,t) \;=\; \inf_{x \le y}
\Big\{\, t_j\, \KLiU(\nu_j,x) \;+\; t_a\, \KLiL(\nu_a,y) \,\Big\}.
\]
The infimum above is attained at a common point between the EVaR of the two distributions, whence the above equals 
%\cmt{ Need a proposition or lemma ?}
\[
g_{a,j}(\nu,t) \;=\; \inf_{x \in [\,\EVaR(\nu_j),\, \EVaR(\nu_a)\,]}
\Big\{\, t_j\, \KLiU(\nu_j,x) \;+\; t_a\, \KLiL(\nu_a,x) \,\Big\}.
\]
Let $\nu=(\nu_1,\dots,\nu_K)\in \Pc([0,1])^K$ and let $\Sigma_K:=\{t\in\mathbb R_+^K:\sum_{k=1}^K t_k=1\}$.
Define, for a fixed candidate best arm $j$,
\begin{equation}
\label{eq:min-min-g}
\Phi(\nu,t)
\;:=\;
\min_{a\ne j}\ \underbrace{\min_{x\in I_{a,j}(\nu)}
\Big\{\, t_j\,\KLiU(\nu_j,x)\ +\ t_a\,\KLiL(\nu_a,x)\,\Big\}}_{:=\ g_{a,j}(\nu,t)}\,,
\end{equation}
where $I_{a,j}(\nu):=\big[\EVaR_\alpha(\nu_j),\,\EVaR_\alpha(\nu_a)\big]$.
Consider the allocation program
\begin{equation}\label{eq:Tmu}
V(\nu)\;=\;\max_{t\in\Sigma_K}\ \Phi(\nu,t),
\qquad
t^*(\nu)\ :=\ \arg\max_{t\in\Sigma_K}\ \Phi(\nu,t).
\end{equation}
We want to show that $t^*(\nu)$ is nonempty, compact, and convex. Moreover, the correspondence
$\nu\mapsto t^*(\nu)$ is upper hemicontinuous on $\Pc([0,1])^K$ (with the product weak topology).

\begin{corollary}
\label{cor:joint-cont-g}
    By Lemma~\ref{lem:evar-cont}, we have $\nu\mapsto \EVaR_\alpha(\nu_k)$ continuous on $\Pc([0,1])$; hence
the interval map $\nu\mapsto I_{a,j}(\nu)$ is continuous (endpoints vary continuously).
By Lemma~\ref{lem:detailed-levelset} the EVaR sub/superlevel sets are compact, which yields
\emph{attainment} in the defining projections for $\KLiL$ and $\KLiU$. The dual formulas imply that
$\KLiL(\cdot,\cdot)$ and $\KLiU(\cdot,\cdot)$ are continuous in both arguments on
$\Pc([0,1])\times(0,1)$ (See Lemma~\ref{lem:KLiLKLiU-cont}).
With joint continuity in $(\nu,x)$ of the inner term, and a compact interval of thresholds $I_{a,j}(\nu)$, Berge’s minimum theorem gives continuity of
 $g_{a,j}$  in $(\nu,x)$.
\end{corollary}

We are ready to present the proof of Lemma~\ref{lem:stab_EVaR}
\begin{proof}[Lemma~\ref{lem:stab_EVaR}]
\label{prf:lem:stab_EVaR}
We now aim to show $\Phi(\nu,\cdot)$ is concave and continuous on $\Sigma_K$:
Fix $(a,j)$.
For each $x$, the map $t\mapsto t_j\,\KLiU(\nu_j,x)+t_a\,\KLiL(\nu_a,x)$ is affine (hence continuous) in $t$.
Therefore $g_{a,j}(\nu,\cdot)$, being the pointwise infimum over $x\in I_{a,j}(\nu)$ of affine functions, is
\emph{concave} and upper semicontinuous.
by Berge’s minimum theorem applied to the compact interval
$I_{a,j}(\nu)$ and the continuity in $(\nu,x)$ by Corollary~\ref{cor:joint-cont-g}, it is in fact \emph{continuous} in $(\nu,t)$.
Since $\Phi(\nu,\cdot)$ is the pointwise minimum over finitely many concave, continuous $g_{a,j}$’s,
$\Phi(\nu,\cdot)$ is concave and continuous on $\Sigma_K$.

For fixed $\nu$, $\Sigma_K$ is compact and convex, and $\Phi(\nu,\cdot)$ is continuous.
By Weierstrass, the maximum in \eqref{eq:Tmu} is \emph{attained}; hence $t^*(\nu)$ is nonempty and compact.
Because $\Phi(\nu,\cdot)$ is concave on the convex set $\Sigma_K$, its upper level sets
$\{t\in\Sigma_K:\ \Phi(\nu,t)\ge c\}$ are convex for every $c$.
In particular, $t^*(\nu)=\{t\in\Sigma_K:\ \Phi(\nu,t)\ge V(\nu)\}$ is convex.
We apply Berge’s maximum theorem (argmax correspondence). Since
(i) $\Phi$ is continuous on the product $\Pc([0,1])^K\times \Sigma_K$,
and (ii) the feasible set $\Sigma_K$ is compact (and independent of $\nu$),
it follows that the value function $\nu\mapsto V(\nu)$ is continuous and the argmax
correspondence $\nu\mapsto t^*(\nu)$ is nonempty, compact-valued, and \emph{upper hemicontinuous}.
\end{proof}
\subsection{Proof of Dual Representations}
\label{proof:prp:kliU}
In this section, we provide the proofs related to Section~\ref{sec:lowerbound}. First, 
we provide the proof of Theorem~\ref{Th:KLiU}.

\subsubsection{Proof of Theorem~\ref{Th:KLiU}}
\begin{proof}
   The  Lagrangian is
\begin{equation}
\label{eq:L}
\begin{aligned}
\mathcal{L}(\kappa,Q;\lambda)
&= \KL(\eta\Vert \kappa)
\;+\;
\lambda_2\Big(\!\int \dd Q -1\Big)
\;+\;\lambda_1\big(\nu-\E_Q[X]\big)
\;+\;\lambda_3\big(\KL(Q\Vert \kappa)-\rho\big)
\\
&\quad+\;\lambda_4\Big(\!\int \dd\kappa -1\Big)
\end{aligned}
\end{equation}
with following multipliers  $\boldsymbol{\lambda}:=( \lambda_1\ge 0, \lambda_2\in \mathbb{R}, \lambda_3\ge 0, \lambda_4\in\mathbb{R})$.

We first regroup the terms dependent on $Q$ and find the minimizer for each $\kappa$
\begin{align}
I(Q)&:=\lambda_3 \KL(Q||\kappa)-\lambda_1 \E_Q[X] + \lambda_2 \int dQ =\lambda_3\Big( \KL(Q||\kappa) - \E_Q[\frac{\lambda_1}{\lambda_3} X - \frac{\lambda_2}{\lambda_3}]\Big)
\end{align}
For $\lambda_3>0$ and for each $\kappa$, by using Donsker-Varadhan \cite{} we obtain the minimum of $I(Q^*)$ and the minimizer:
\begin{align}
\inf_{Q} I(Q)&:= \lambda_3(\frac{\lambda_2}{\lambda_3} - \log \E_{\kappa}[\exp(\frac{\lambda_1}{\lambda_3}X)])
\end{align} where
\begin{equation}
    \frac{dQ^*}{d\kappa}(x) = \frac{\exp(\frac{\lambda_1}{\lambda_3}x)}{\E_\kappa[\exp(\frac{\lambda_1}{\lambda_3}X)]}
\end{equation}
After eliminating $Q$,  the Lagrangian is reduced to 

\begin{align}
\mathcal{L}(\kappa,Q;\lambda)
&= \KL(\eta\Vert \kappa)
\;+\;
\lambda_2\Big( -1\Big)
\;+\;\lambda_1 \nu
\;+\;\lambda_3\big(-\rho\big)
+\;\lambda_4\Big(\!\int \dd\kappa -1\Big)
\\
&+\; \lambda_2 -\lambda_3 \log \E_{\kappa}[\exp(\frac{\lambda_1}{\lambda_3}X)]
\end{align}
Now, we take directional derivative in $\kappa$ and check the first order optimality condition.

Let $\kappa^*$ be a minimizer of $\kappa\mapsto \mathcal{L}(\kappa;\boldsymbol{\lambda})$ (for fixed multipliers).
For any $\kappa_1\in \Mcal$, consider the convex segment
\[
\kappa_t \;=\; (1-t)\,\kappa^* + t\,\kappa_1,\qquad t\in[0,1].
\]
For KL term we have:
\[
\left.\frac{d}{dt}\right|_{0}\int \log\!\Big(\frac{d\eta}{d\kappa_t}\Big)\,d\eta
=\int \frac{d\eta}{d\kappa^*}(x)\,d\!\big(\kappa^*-\kappa\big)(x).
\]
We define $Z_{\kappa}(\frac{\lambda_1}{\lambda_3}):= \E_\kappa[\exp(\frac{\lambda_1}{\lambda_3}X)]$ term
$Z_{\kappa_s}(\cdot)=(1-t)Z_{\kappa^*}(\cdot)+t Z_{\kappa}(\cdot)$, hence
\[
\left.\frac{d}{dt}\right|_{0}\big[-\lambda_3\log Z_{\kappa_s}(\cdot)\big]
=-\lambda_3\,\frac{Z_{\kappa}(\cdot)-Z_{\kappa^*}(\cdot)}{Z_{\kappa^*}(\cdot)}
=\int \frac{\lambda_3\,e^{\frac{\lambda_1}{\lambda_3} x}}{Z_{\kappa^*}(\cdot)}\,d\!\big(\kappa^*-\kappa\big)(x).
\]
and
\[
\left.\frac{d}{dt}\right|_{0}\lambda_4\!\int d\kappa_s
=\lambda_4\Big(\int d\kappa-\int d\kappa^*\Big)
=\int (-\lambda_4)\,d\!\big(\kappa^*-\kappa\big)(x).
\]
Now, we combine all terms and the directional derivative is
\begin{equation}\label{eq:dirderiv}
\left.\frac{d}{dt}\right|_{0} L(\cdot)
=\int \Big[\,\frac{d\eta}{d\kappa^*}(x)+\frac{\tau\,e^{t x}}{Z_{\kappa^*}(t)}-\lambda_4\Big]\,
d\!\big(\kappa^*-\kappa\big)(x).
\end{equation}
By convexity, $\kappa^*$ is the (unique) minimizer iff the derivative is $\ge 0$ for every $\kappa$, i.e.
\begin{equation}\label{eq:KKT}
\frac{d\eta}{d\kappa^*}(x)+\lambda_3\frac{e^{\frac{\lambda_1}{\lambda_3} x}}{Z_{\kappa^*}(\frac{\lambda_1}{\lambda_3})}-\lambda_4
=
\begin{cases}
0, & x\in \operatorname{supp}(\kappa^*),\\[2pt]
\ge 0, & \text{elsewhere.}
\end{cases}
\end{equation}
Define
$a:=\frac{\lambda_1}{\lambda_3},\quad
Z_{\kappa^*}(a):=\int e^{a u}\,d\kappa^*(u)$, 
we have \(\kappa^*\)-a.s.
\[
\frac{d\eta}{d\kappa^*}(x)=\lambda_4-\lambda_3\frac{e^{a x}}{Z_{\kappa^*}(a)}.
\]
 Using Radon-Nikodým and integrate both sides against \(d\kappa^*\).
\[
1= \int d\eta = \int \frac{d\eta}{d\kappa^*}\,d\kappa^*=\lambda_4 -\lambda_3 \int \frac{e^{a x}}{Z_{\kappa^*}(a)}\,d\kappa^*(x)
=\lambda_4 -\lambda_3\frac{1}{Z_{\kappa^*}(a)}\int e^{a x}\,d\kappa^*(x)
\]
where \(\int d\kappa^*=1\) and $Z_{\kappa^*}(a) = \int e^{a x}\,d\kappa^*(x)$. Hence $\lambda_4=1+\lambda_3$.
Substituting \(\lambda_4=1+\lambda_3\) back gives, \(\kappa^*\)-a.s.,
\[
\quad
\frac{d\eta}{d\kappa^*}(x)
=1+\lambda_3\Bigg(1-\frac{e^{a x}}{Z_{\kappa^*}(a)}\Bigg).
\]
We now eliminate \(Z_{\kappa^*}(a)\). From the \(Q\)-side optimality (Donsker-Varadhan tilt) we have
\[
\frac{dQ}{d\kappa^*}(x)=\frac{e^{a x}}{Z_{\kappa^*}(a)}\,,\qquad a=\frac{\lambda_1}{\lambda_3},
\]
and complementary slackness at optimum (both constraints tight):
\[
\mathrm{KL}(Q\Vert\kappa^*)=\rho,\qquad \E_Q[X]=\nu.
\]
Therefore
\[
\rho
=\int \log\!\Big(\frac{dQ}{d\kappa^*}\Big)\,dQ
=\int \Big(a x-\log Z_{\kappa^*}(a)\Big)\,dQ(x)
=a\,\nu-\log Z_{\kappa^*}(a),
\]
so
\[
\quad
Z_{\kappa^*}(a)=\E_{\kappa^*}\big[e^{a X}\big]=\exp\big(a\,\nu-\rho\big).
\quad
\]

\paragraph{Final dual form.}
Plugging \(Z_{\kappa^*}(a)=e^{a\nu-\rho}\) into the \(\kappa^*\)-a.s. stationarity gives
\[
\frac{d\eta}{d\kappa^*}(x)
=1+\lambda_3\Big(1-e^{a(x-\nu)+\rho}\Big).
\]
Substituting this back into the reduced Lagrangian cancels the linear terms and yieldt the dual
\[
\;
\sup_{\lambda_3\ge 0,\ \lambda_1\ge 0}\ 
\E_{\eta}\!\left[
\log\!\Big(1+\lambda_3\big(1-\exp\!\big(\tfrac{\lambda_1}{\lambda_3}(X-\nu)+\rho\big)\big)\Big)
\right],
\;
\]
with the implicit feasibility (positivity under the log) for all \(x\) in the support of \(\eta\):
\[
1+\lambda_3\Big(1-\exp\!\big(\tfrac{\lambda_1}{\lambda_3}(x-\nu)+\rho\big)\Big)\;>\;0.
\]
Here \(\lambda_3\ge 0\) is the multiplier of the KL constraint and \(\lambda_1\ge 0\) that of the mean constraint \( \E_Q[X]\ge\nu\). The ratio \(a=\lambda_1/\lambda_3\) is well-defined at the (nontrivial) optimum where \(\lambda_3>0\); the boundary case \(\lambda_3=0\) correspondt to \(Q=\kappa^*\) and gives zero objective when \(\nu\le \E_{\kappa^*}[X]\).
%
%region simplification
The union with \(\{\lambda_3=0\}\) covers the boundary case (for which the dual
integrand equals \(\log 1=0\)). For \(\lambda_3>0\), the constraint can be written equivalently as
\[
\sup_{x\in\supp(\eta)}\ \lambda_3\bigl(\exp(\tfrac{\lambda_1}{\lambda_3}(x-\nu)+\rho)-1\bigr)\ <\ 1.
\]
For \(X\in[0,1]\), then for \(a:=\lambda_1/\lambda_3\ge 0\) the function
\(x\mapsto\exp(a(x-\nu))\) is nondecreasing, so the worst case is at
\(x_{\max}:=\sup\supp(\eta)\) (we can assume \(x_{\max}=1\)).

\end{proof}
\begin{remark}
In the CVaR $\KLiU$ primal there is a \emph{pointwise domination} $0\le dW\le d\kappa$, which requires a functional multiplier $\lambda_5(x)$ in the Lagrangian. EVaR uses a \emph{global} KL-ball; absolute continuity $Q\ll\kappa$ is already enforced by the KL term, and there is no $dQ\le d\kappa$ constraint. Hence no $\lambda_5(\cdot)$ appears here.
\end{remark}

\subsubsection{Proof of Theorem~\ref{Th:KLiL}}\label{sec:prf:KLiL}

\begin{proof}
For any $\kappa$ and $\nu$, we write
\[
\EVaR_\alpha(\kappa)\le \nu
\iff \exists\,z>0:\ \log \E_\kappa[e^{z X}] \le -\rho+z\nu
\iff \exists\,z>0:\ \int e^{z x}\,d\kappa(x) \le e^{-\rho+z\nu}.
\]
Hence
\begin{equation}
\label{eq:KLL-outer-beta}
\KLiL(\eta,\nu)
\ =\
\inf_{z>0}\ 
\inf_{\kappa\in\mathcal P([0,1])\ :
\E_\kappa [\exp(z X)]\le \exp({-\rho+z\nu})}
\ \KL(\eta\Vert\kappa).
\end{equation}

For fixed $z>0$, we introduce multipliers $\lambda\ge 0$ (for the exponential moment) and $\tau\in\R$ (for normalization).
The Lagrangian is
\begin{align}
\mathcal L(\kappa;\tau,\lambda)
&= \int \log\!\Big(\frac{d\eta}{d\kappa}\Big)\,d\eta
 + \tau\Big(\int d\kappa - 1\Big)
 + \lambda\Big(\int e^{z x}\,d\kappa - e^{-\rho+z\nu}\Big).
\label{eq:Lag}
\end{align}

We will derive the first-order optimality conditions for a minimizer $\kappa^*$ of
\(\kappa\mapsto\mathcal L(\kappa;\tau,\lambda)\) (for fixed $\tau,\lambda$) using
the \emph{mixture path}
\[
\kappa_t := (1-t)\kappa^*+t\kappa,\qquad t\in[0,1],
\]
where $\kappa$ is any other feasible probability measure on $[0,1]$.

Since both $\kappa$ and $\kappa^*$ are probability measures,
\[
\frac{d}{dt}\,\tau\int d\kappa_t\Big|_{t=0}
=\tau\int d(-\kappa^* +\kappa)
\]
For the exponential-moment term,
\[
\frac{d}{dt}\,\lambda\int e^{z x}\,d\kappa_t\Big|_{t=0}
=\lambda\int e^{z x}\,d(-\kappa^* +\kappa)(x).
\]
Write
\[
\int \log\!\Big(\frac{d\eta}{d\kappa_t}\Big)\,d\eta
\ =\ -\int \log\!\Big(\frac{d\kappa_t}{d\eta}\Big)\,d\eta,
\]
which is finite in a neighborhood of $t=0$ because $\kappa^*$ minimizes the Lagrangian and hence dominates $\eta$ (otherwise the KL is $+\infty$). 
Taking derivative of first KL term:
\[
\int\log\!\Big(\frac{d\eta}{d\kappa_t}\Big)\,d\eta
\ =\ -\int \log\!\Big(u_t(x)\Big)\,d\eta(x),
\qquad
u_t(x):=\frac{d\kappa_t}{d\eta}(x).
\]
Because $\kappa_t=(1-t)\kappa^*+t\kappa$ is affine in $t$, we have
$u_t=(1-t)u_0+t\,u_1$ with $u_0=\frac{d\kappa^*}{d\eta}$ and
$u_1=\frac{d\kappa}{d\eta}$. Hence, by the chain rule and at $t=0$
\[
-\int \frac{u_1(x)-u_0(x)}{u_0(x)}\,d\eta(x)
=-\int \frac{d\kappa-d\kappa^*}{d\eta}(x)\,
     \frac{d\eta}{d\kappa^*}(x)\,d\eta(x)
=-\int \frac{d\eta}{d\kappa^*}(x)\,d(\kappa-\kappa^*)(x),
\]
where the last equality is the change-of-measure identity.
Collecting terms,
\begin{equation}\label{eq:dir-deriv-final}
=\int \Big[-\,\frac{d\eta}{d\kappa^*}(x)\;+\;\lambda e^{z x}+\tau\Big]\,
   d(\kappa-\kappa^*)(x).
\end{equation}

Now, we need to chaeck First-order optimality.
Since $\kappa^*$ minimizes $\kappa\mapsto \mathcal L(\kappa;\tau,\lambda)$ over the convex set of probabilities,
we require $\Lc'(0)\ge0$ for all admissible $\kappa$. Variations $d(\kappa-\kappa^*)$
have total mass $0$ and can move arbitrarily small mass between measurable sets, hence
there exists a constant $\tau\in\mathbb R$ such that
\[
-\,\frac{d\eta}{d\kappa^*}(x)\;+\;\tau\;+\;\lambda e^{z x}\ \ge\ 0\quad\forall x\in[0,1],
\qquad
\Big(-\,\frac{d\eta}{d\kappa^*}(x)\;+\;\tau\;+\;\lambda e^{z x}\Big)\,d\kappa^*(x)=0.
\]
Equivalently, on $\mathrm{supp}(\kappa^*)$,
\[
\frac{d\eta}{d\kappa^*}(x)=\tau+\lambda e^{z x}
\quad\Longleftrightarrow\quad
\frac{d\kappa^*}{d\eta}(x)=\frac{1}{\tau+\lambda e^{z x}}.
\]
The remaining KKT conditions are
\[
\int \frac{1}{\tau+\lambda e^{z x}}\,d\eta(x)=1,
\qquad
\lambda\!\left(\int e^{z x}\,d\kappa^*(x)-e^{-\rho+z\nu}\right)=0,\ \ \lambda\ge0.
\]
To simplify more, we use normalization and complementary slackness conditions
\begin{align}
&\int d\kappa^*(x)=1
\quad\Longleftrightarrow\quad
\int \frac{1}{\tau+\lambda e^{z x}}\,d\eta(x)=1,
\label{eq:norm}\\
&\lambda\Big(\int e^{z x}\,d\kappa^*(x)\;-\;e^{-\rho+z\nu}\Big)=0,\qquad \lambda\ge0,
\label{eq:slack}
\end{align}
so if $\lambda>0$ the moment constraint is tight:
\(\displaystyle \int e^{z x}\,d\kappa^*(x)=e^{-\rho+z\nu}\).
Using these conditions yields
\( \tau = 1-\lambda e^{z \nu - \rho}\). Put back to the optimial minimizer:

\[
\frac{d\kappa^*}{d\eta}(x)=\frac{1}{1-\lambda ( e^{z \nu - \rho} - e^{z x})}.
\]
as we expected  for the inner problem, the form of $\frac{d\kappa^*}{d\eta}(x)$ is similar to have expectation of $e^{z X}$.
Moreover, the dual form of $\KLiL$ is

\[\min_{z \in (0, s]}\max_\lambda \E_{\eta}[\log(1-\lambda ( e^{z \nu - \rho} - e^{z x}))]. \]

\subsection*{Feasibility of $\lambda$ by positivity under the logarithm}

For fixed $z>0$, the dual integrand is
\(
\log\!\Big(1-\lambda\big(e^{-\rho+z\nu}-e^{z X}\big)\Big),
\)
so we must enforce
\begin{equation}\label{eq:feas-ptwise}
1-\lambda\big(e^{-\rho+z\nu}-e^{z x}\big)\;>\;0
\qquad\text{for all }x\in[0,1].
\end{equation}
Since $X\in[0,1]$, we have $e^{z X}\in[1,e^{z}]$ and the quantity
$e^{-\rho+z\nu}-e^{z x}$ attains its \emph{largest positive} value at $x=0$.
Therefore \eqref{eq:feas-ptwise} is equivalent to the uniform bound
\[
\lambda\;<\;\frac{1}{\displaystyle \sup_{x\in[0,1]}\big(e^{-\rho+z\nu}-e^{z x}\big)_+}
\qquad\text{with the convention } \frac{1}{0}=+\infty,
\]
i.e.,
\begin{equation}\label{eq:lambda-domain}
\quad
\lambda\ \in\ \mathcal D(z,\nu)\ :=\
\begin{cases}
[\,0,\ \infty)\,, & \text{if } e^{-\rho+z\nu}\le 1,\\[3pt]
\big[\,0,\ \ \dfrac{1}{\,e^{-\rho+z\nu}-1\,}\ \big)\,, & \text{if } e^{-\rho+z\nu}>1.
\end{cases}
\end{equation}
(Here $(\cdot)_+=\max\{\,\cdot,0\}$ and the interval is \emph{open} at the right endpoint to keep the log’s argument strictly positive.)

\paragraph{Remark.}
If $e^{-\rho+z\nu}\le 1$, then $e^{-\rho+z\nu}-e^{z x}\le 0$ for all $x\in[0,1]$, so the factor in the log is $\ge 1$ for every $\lambda\ge 0$ and there is no upper bound from feasibility. In that regime the inner supremum in $\lambda$ can diverge; in the overall program
\[
\KLiL(\eta,\nu)=\inf_{z>0}\ \sup_{\lambda\in\mathcal D(z,\nu)}
\ \E_\eta\!\Big[\log\big(1-\lambda\big(e^{-\rho+z\nu}-e^{z X}\big)\big)\Big],
\]
the outer infimum over $z$ selects values with $e^{-\rho+z\nu}>1$ (i.e., $z>\rho/\nu$ when $\nu>0$) to yield a finite optimum.

\end{proof}

\subsection{Proof of Joint Continuity of the KL-projection functionals}
\label{app:prf:lem:KLiLKLiU-cont}
Here, we give the proof of Lemma~\ref{lem:KLiLKLiU-cont}.

We need the joint continuity of $\KLiL$ and $\KLiU$ in two places; one in oracle continuity via Berge and when we reduce the search space with attainment in $[\EVaR(\nu^*),\EVaR(\nu)]$. We give the formal proof here.
%{\color{red} need to be replaced to appendix, and I need it before Lemma~9 !!}
%\subsection{KL-projections' Properties}

\begin{remark}
\label{rem:kl-cont}
%[Lower semicontinuity of KL]
On a Polish space (in particular on $[0,1]$), the relative entropy
$(\eta,\kappa)\mapsto \KL(\eta\Vert\kappa)$ is lower semicontinuous with respect
to weak convergence. This is classical and follows from the Donsker-Varadhan
(Gibbs) variational formula
\[
\KL(\eta\Vert\kappa)=\sup_{f\in C_b}\{\E_\eta[f]-\log \E_\kappa[e^{f}]\},
\]
as a supremum of continuous functionals is l.s.c. \cite{DemboZeitouni1998}.
\end{remark}

We give the proof for $\KLiL$. The proof for  $\KLiU$ case is similar. First, we study joint lower semicontinuity, and later present joint upper semicontinuity.

\begin{lemma}
\label{lem:EVaR-KLL-lsc}
For $\eta \in \in \Pc([0,1])$ and $\nu \in (0,1)$, the functionals $\KLiU(\eta,\nu)$ and $\KLiL(\eta,\nu)$ are jointly lower-semicontinuous in $(\eta,\nu)$.
\end{lemma}
\begin{proof}
 \label{prf:joint:KLL_lsc} 
Fix a sequence $(\eta_n,x_n)\to(\eta,x)$ with $\eta_n\Rightarrow\eta$ and $x_n\to x$.
We prove
\[
\liminf_{n\to\infty}\ \KLiL(\eta_n,x_n)\ \ge\ \KLiL(\eta,x).
\]

\medskip

On $[0,1]$, the space $\Pc([0,1])$ is compact under weak convergence (Prokhorov).
By the continuity of $\EVaR_\alpha$ on $\Pc([0,1])$ (proved in Lemma~\ref{lem:evar-cont}), the sublevel set
\[
L(x)\ :=\ \{\kappa\in\Pc([0,1]):\ \EVaR_\alpha(\kappa)\le x\}
\]
is closed, hence $L(x)$ is compact.
Similarly, $L(x_n)$ is compact for every $n$.
Since $\kappa\mapsto \KL(\eta\Vert\kappa)$ is lower semicontinuous under weak convergence
(Gibbs-Donsker-Varadhan variational formula), the minimum
\[
\KLiL(\eta_n,x_n)\ =\ \inf_{\kappa\in L(x_n)} \KL(\eta_n\Vert\kappa)
\]
is \emph{attained} for each $n$. We then pick a minimizer $\kappa_n\in L(x_n)$ such that
\[
\KLiL(\eta_n,x_n)=\KL(\eta_n\Vert\kappa_n).
\]

By compactness of $\Pc([0,1])$, the sequence $(\kappa_{m})$ is relatively compact.  Passing to a subsequence, assume $\kappa_n\Rightarrow\kappa$ for some $\kappa\in\Pc([0,1])$.
By continuity of $e_\alpha$, we have
\[
e_\alpha(\kappa)\ =\ \lim_{n\to\infty} e_\alpha(\kappa_n)\ \le\ \lim_{n\to\infty} x_n\ =\ x,
\]
so $\kappa\in L(x)$ means we checked limit of feasible measures.
\medskip
By joint lower semicontinuity of $\KL$ in $(\eta,\kappa)$ under weak convergence,
\[
\liminf_{n\to\infty}\ \KLiL(\eta_n,x_n)
\ =\
\liminf_{n\to\infty}\ \KL(\eta_n\Vert\kappa_n)
\ \ge\
\KL(\eta\Vert\kappa)
\ \ge\
\inf_{\tilde\kappa\in L(x)} \KL(\eta\Vert\tilde\kappa)
\ =\ \KLiL(\eta,x).
\]
This proves the claimed lower semicontinuity.
\end{proof}

\begin{lemma}\label{lem:EVaR-KLL-usc}
The map $(\eta,x)\mapsto \KLiL(\eta,x)$ is jointly upper semicontinuous on $\Pc([0,1])\times[0,1]$.
\end{lemma}
To prove u.s.c, we divide the region into three subsets $\{x > e_\alpha \}$, $\{x < e_\alpha \}$ and $\{x = e_\alpha \}$ and examine each one separately.
Notice that $\EVaR_\alpha(\eta) = 1$ iff $\eta\{X=1\}\ \ge\ 1-\alpha$ with the minimizer  $z^*(\eta)=+\infty$. We need to address this just in $\{x = e_\alpha \}$ while it is impossible to happen in the first case  $\{x > e_\alpha \}$ ($x\in [0,1]$), and in the second case $\{x < e_\alpha \}$ ( $e_\alpha\leq 1$). This clarification helps us to be able to consider  a compact set of the minimizer.

\begin{proof}
\label{prf:lem:EVaR-KLL-usc}
\paragraph{ 1) Case $x>e_\alpha(\eta)$:}

 If $x>e_\alpha(\eta)$ then $\eta$ itself is feasible for the
constraint $e_\alpha(\kappa)\le x$, hence $\KLiL(\eta,x)=0$.
Let $(\eta_n,x_n)\to(\eta,x)$ with $\eta_n\Rightarrow\eta$ and $x_n\to x$. By continuity of $e_\alpha$ in the law
(Lemma~\ref{lem:evar-cont}), $e_\alpha(\eta_n)\to e_\alpha(\eta)$. Set
\[
\varepsilon:=\frac{x-e_\alpha(\eta)}{3}>0.
\]
For all $n$ large enough we have
\[
|x_n-x|<\varepsilon
\quad\text{and}\quad
\big|e_\alpha(\eta_n)-e_\alpha(\eta)\big|<\varepsilon,
\]
hence
\[
x_n \;\ge\; x-\varepsilon \;=\; e_\alpha(\eta)+2\varepsilon
\;>\; e_\alpha(\eta)+\varepsilon \;\ge\; e_\alpha(\eta_n).
\]
Therefore $e_\alpha(\eta_n)\le x_n$, so $\eta_n$ is feasible for $\KLiL(\eta_n,x_n)$ and
\[
0\ \le\ \KLiL(\eta_n,x_n)\ \le\ \KL(\eta_n\Vert\eta_n)\ =\ 0,
\]
i.e.\ $\KLiL(\eta_n,x_n)=0$ for all large $n$. Consequently,
\[
\limsup_{n\to\infty}\KLiL(\eta_n,x_n)=0=\KLiL(\eta,x),
\]
which proves continuity (hence upper semicontinuity) of $\KLiL$ in the region $x>e_\alpha(\eta)$.

\paragraph{2) Case $x<e_\alpha(\eta)$:}
By the EVaR-KLL dual (Theorem~\ref{Th:KLiL}), we can write
\[
\KLiL(\eta,x)
\;=\;\inf_{z\in \mathcal Z(x)}\ h^\ast(z,x,\eta),
\qquad
h^\ast(z,x,\eta):=\sup_{\lambda\in\mathcal D(z,x)}
\ \E_\eta\!\left[\log\!\big(1-\lambda\big(e^{-\rho+zx}-e^{zX}\big)\big)\right],
\]
where the outer set $\mathcal Z(x)=[\,z_{\min}(x),z_{\max}(x)\,]$ is a compact
interval contained in $(\rho/x,\infty)$ with endpoints continuous in $x$, and the inner feasible set is the
(one-dimensional) compact interval
\[
\mathcal D(z,x)\;=\;\left[\,0,\ \frac{1}{\,e^{-\rho+zx}-1\,}\right].
\]

To invoke Berge’s maximum theorem we must ensure two ingredients: a \emph{continuous} criterion on
 a \emph{compact}, upper–hemicontinuous feasible set.
Once these two ingredients are verified for the EVaR dual, the inner value
$h^\ast(z,x,\eta)$ is jointly upper semicontinuous (indeed continuous) in $(\eta,z,x)$. 
\begin{lemma}
\label{lem:cont-2case}
    Let $\mathcal Z\subset(0,\infty)$ and $I\subset(0,1)$ be compact intervals.
Then there exists $\varepsilon>0$ such that, with the shrunken feasible set
\[
\mathcal D_\varepsilon(z,x):=\Big[\,0,\ \frac{1}{\,e^{-\rho+zx}-1\,}-\varepsilon\Big],
\]
the map
\[
(\eta,z,x)\ \longmapsto\ \sup_{\lambda\in\mathcal D_\varepsilon(z,x)}\ \E_\eta[\phi(X;z,x,\lambda)]
\]
is \emph{continuous} (hence u.s.c.) on $\Pc([0,1])\times\mathcal Z\times I$ (weak topology on $\Pc([0,1])$).
\end{lemma}

\begin{proof}
\label{prf:lem:cont-2case}
    On the compact set $[0,1]\times\mathcal Z\times I$ the upper endpoint
$\Lambda(z,x):=(e^{-\rho+zx}-1)^{-1}$ is finite and continuous. Choose $\varepsilon\in(0,1)$
so small that $\Lambda(z,x)-\varepsilon>0$ for all $(z,x)\in\mathcal Z\times I$.
Then for $(u,z,x,\lambda)\in[0,1]\times\mathcal Z\times I\times[0,\Lambda(z,x)-\varepsilon]$ the
factor inside the logarithm stays in $(0,1]$ uniformly, hence
\[
(u,z,x,\lambda)\ \mapsto\ \phi(u;z,x,\lambda)
\]
is \emph{bounded and continuous} on that compact parameter box.

Let $\eta_n\Rightarrow\eta$ and $(z_n,x_n,\lambda_n)\to(z,x,\lambda)$ with
$\lambda_n\in\mathcal D_\varepsilon(z_n,x_n)$.
Then $\phi(\cdot;z_n,x_n,\lambda_n)\to \phi(\cdot;z,x,\lambda)$ uniformly on $[0,1]$
(continuity on a compact set). Hence
\[
\Big|\E_{\eta_n}[\phi(X;z_n,x_n,\lambda_n)]-\E_{\eta_n}[\phi(X;z,x,\lambda)]\Big|\ \to\ 0,
\]
and by Portmanteau (bounded continuous test function),
\(
\E_{\eta_n}[\phi(X;z,x,\lambda)]\to \E_{\eta}[\phi(X;z,x,\lambda)].
\)
Thus
\[
\E_{\eta_n}[\phi(X;z_n,x_n,\lambda_n)]\ \longrightarrow\ \E_{\eta}[\phi(X;z,x,\lambda)].
\]
The set-valued map $(z,x)\mapsto \mathcal D_\varepsilon(z,x)$ has compact values and continuous
endpoints, hence a closed graph and is upper hemicontinuous on $\mathcal Z\times I$.
\end{proof}

By Lemma~\ref{lem:cont-2case}, the criterion is continuous and the feasible correspondence is compact u.h.c.;
therefore Berge’s maximum theorem implies
\[
(\eta,z,x)\ \longmapsto\ \sup_{\lambda\in\mathcal D_\varepsilon(z,x)} \E_\eta[\phi(X;z,x,\lambda)]
\]
is \emph{continuous} on $\Pc([0,1])\times\mathcal Z\times I$.
The proof for u.s.c. of $\KLiL(\eta, x)$ for $x<e_\alpha(\eta)$.

\paragraph{3) Case $x=e_\alpha(\eta)$.}
Let $(\eta_n,x_n)\to(\eta,e_\alpha(\eta))$ with $\eta_n\Rightarrow\eta$. 
%(otherwise we are in Case~1)
Without loss of generality, assume that for all large $n$,
\[
x_n\ \le\ e_\alpha(\eta_n).
\]
We aim to prove that
\[
\limsup_{n\to\infty}\KLiL(\eta_n,x_n)\ \le\ \KLiL(\eta,e_\alpha(\eta))\ =\ 0.
\]

We split into two subcases according to the boundary level $e_\alpha(\eta)$.

\textbf{Subcase A: $e_\alpha(\eta)>0$.}
Define
\(
w_n\ :=\
\begin{cases}
0, & e_\alpha(\eta_n)=0,\\[3pt]
1-\dfrac{x_n}{e_\alpha(\eta_n)}, & e_\alpha(\eta_n)>0,
\end{cases}
\), so $w_n\in[0,1]$. Since $x_n\to e_\alpha(\eta)$ and $e_\alpha(\eta_n)\to e_\alpha(\eta)$, we have $w_n\ \longrightarrow\ 0$. We also define $\kappa_n\ :=\ (1-w_n)\,\eta_n\ +\ w_n\,\delta_0.$
By convexity and $\EVaR_\alpha(\delta_0)=0$,
\[
\EVaR_\alpha(\kappa_n)\ \le\ (1-w_n)\EVaR_\alpha(\eta_n)
=(1-w_n)e_\alpha(\eta_n)=x_n,
\]
so $\kappa_n$ is feasible for $\KLiL(\eta_n,x_n)$.

The simple mass-shiftting mixture introduced by $\kappa_n$ gives a clean KL bound when $e_\alpha(\eta)>0$ (because $w\to 0$), but it can fail when $e_\alpha(\eta)=0$ (it may force $w$ away from $0$, thus KL blows up). That is why we introduce a separate subcase for  $e_\alpha(\eta)=0$ boundary.

\textbf{Subcase B: $e_\alpha(\eta)=0$.}
\fix{ In risk measures similar to CVaR \cite{agrawal2021BAIcvar}, a mass-shifting method covers this special case, while in EVaR, we have to be careful  specially  in controlling the KL-projections, here, if $e_\alpha(\eta)=0$, $w_n$ does not need to vanish, so $\KL(\eta_n, \kappa_n)$ can blow up.
We use here \emph{exponential change of measure} or \emph{Esscher tilt} method to tackle this problem.
\begin{lemma}
    \label{lem:Esscher}
    Let $X\in[0,1]$, $\eta\in\Pc([0,1])$, and for $t\geq 0$, define the tilted law
    \[
    \frac{d \kappa^{(t)}}{d \eta } = \frac{e^{-tX}}{\E_\eta[e^{-tX}]}.
    \]
    Then
    \[
\KL(\eta, \kappa^{(t)}) = \log(\E_\eta[e^{-tX}]) + t \E_\eta[X] \leq t
    \]
\end{lemma}
\begin{proof}
Since $d\eta/d\kappa^{(t)}(x)=e^{t x}\,\E_\eta[e^{-tX}]$, we compute
\[
\KL(\eta\Vert \kappa^{(t)})
=\E_\eta\!\left[\log\!\Big(\frac{d\eta}{d\kappa^{(t)}}\Big)\right]
=\E_\eta[tX]+\log \E_\eta[e^{-tX}]
=t\,\E_\eta[X]+\log \E_\eta[e^{-tX}]
\le t,
\]
because $\log \E_\eta[e^{-tX}]\le 0$ for $t\ge 0$ and $X\in[0,1]$.
\end{proof}
We use Lemma~\ref{lem:Esscher}, as $t \to 0$, we have the u.s.c of the objective.
}

Therefore in both subcases,
\(\limsup_{n\to\infty}\KLiL(\eta_n,x_n) \le 0 = \KLiL(\eta,e_\alpha(\eta)),
\)
which establishes joint upper semicontinuity at the boundary and completes the proof of the lemma.
\end{proof}
%{\color{red} do I need joint KL properties? check \ref{cor:joint-cont-g}}

\label{app:Prf:Th:upp_sample}
\section{Related Proofs to Section~\ref{sec:alg:evar}}

\subsection{Appendix: Proof of Theorem~\ref{Th:upp_sample}}
\label{app:subsec:prf:Th:upp_sample}
  \fix{
  In this appendix we establish $\delta$-correctness and asymptotic optimality of the algorithm following the standard blueprint in \cite{garivier2016optimal,agrawal2021BAIcvar}. We begin from the dual representations of the KL-projection terms, build per-side mixture supermartingale tests for each KL-projection, and verify in Lemma~\ref{lem:expconcave-evar} positivity on compact parameter domains together with exp-concavity of the dual integrands. These properties permit the use of the exp-concave aggregation inequality (Lemma~F.1 of \cite{agrawal2021BAIcvar}), which ties the empirical dual objectives to log-mixture supermartingales. Ville’s inequality then delivers the anytime $\delta$-correct error control, and a standard tracking argument shows that the stopping time matches the information-theoretic lower bound asymptotically.
  }

\begin{proof}\label{Prf:Th:upp_sample}
We assume the first arm is the best one, which algorithm fails to declare. Thus, we need to bound the following probability of error 
\begin{align}
    \Pr(\Ec) \leq \sum_{i=2}^K \Pr\Big( \exists t : 
    N_i(t)\KLiU(\hat{\mu}_i(t),\nu_i) + N_1(t)\KLiL(\hat{\mu}_1(t),\nu_1)\geq  \gamma 
    \Big)
    \end{align}
    or equivalently bounding each summand given in the following proposition which is the main technical part of the proof.

\begin{proposition}\label{prop:anytime-evar}
Fix $\alpha\in(0,1)$ and let $\rho:=-\log(1-\alpha)$. For any two arms $i\neq j$ with true laws
$\mu_i,\mu_j$ on $[0,1]$, write $\nu_i:=\EVaR_\alpha(\mu_i)$ and $\nu_j:=\EVaR_\alpha(\mu_j)$.
Then for all $x>0$,
\[
\Pr\!\Big(\exists n\ge 1:\;
N_i(n)\,\KLiU \big(\widehat\mu_i(n),\nu_i\big)
+
N_j(n)\,\KLiL \big(\widehat\mu_j(n),\nu_j\big)
\;\ge\; x + h(n)\Big)
\;\le\; e^{-x},
\]
where
$h(n)\;:=\;3\,\log(n{+}1)+2$.
\end{proposition}
This proposition is analogous to \cite[Lemma~F.1]{agrawal2021BAIcvar}.
\begin{proof}[Proposition~\ref{prop:anytime-evar}]

Let $(\mathcal F_t)_{t\ge 0}$ be the canonical filtration of the bandit process,
$\mathcal F_t:=\sigma\big(\{(A_s,X_s)\}_{s\le t}\big)$, where $A_s$ is the arm pulled at time $s$
and $X_s\in[0,1]$ the observed reward. For each arm $i$, the subsequence
$\{X^i_j\}_{j\ge1}$ of its pulls is i.i.d.\ with law $\mu_i$, and all conditional expectations
below are taken with respect to $\mathcal F_{t-1}$.

For each arm $i$, $X_j^i$ denotes $j\in \{1,\cdots, N_i(n) \}$, we rewrite the empirical form of dual representation given in Theorem~\ref{Th:KLiU} as
\begin{equation}
\label{eq:emp-du-KLU}
N_i(n) \KLiU (\eta,y)
\;=\;
\sup_{\lambda_1\ge 0,\;\lambda_3>0}\;
%\mathbb{E}_{\hat{\eta}}
\sum_{j=1}^{N_i(n)}\!
\log\!\Big(1+\lambda_3\big(1-\exp\!\big[\tfrac{\lambda_1}{\lambda_3}(X_j^i-\nu_i)+\rho\big]\big)\Big).
\end{equation}
The one given in Theorem~\ref{Th:KLiL} is written as
\begin{equation}
\label{eq:emp-du-KLL}
N_i(n) \KLiL (\eta,x)
\;\le \;
 \sup_{\gamma\in D(x, z_0)}\;
%\mathbb{E}_{\hat{\eta}}
\sum_{j=1}^{N_i(n)}\!
\log\!\Big(1-\gamma\big(e^{-\rho+z_0 \nu_i}-e^{z_0 X_j^i}\big)\Big),
\end{equation}
where we use \emph{fixed-slope relaxation} which drops the outer $\inf$ over $z$ by fixing $z_0>0$. This way we reduce the error event to a union of pairwise tests for some time $n$.
The goal is therefore to upper-bound each of these probabilities by $\frac{\delta}{K-1}$.

We define the  multiplicative updates of our test supermartingale for both sides. For upper-side ($EVaR \geq x$) we have
\[M_U(x;\lambda_1,\lambda_3,y)
:=1+\lambda_3\Big(1-e^{(\frac{\lambda_1}{\lambda_3})(x-y)+\rho}\Big),\] where
feasibility requires $M_U(x;\lambda_1,\lambda_3,y)\ge 0$ for all $x\in[0,1]$. Hence, we set
\[{\mathcal R}_U(y)
:=\Big\{(\lambda_1,\lambda_3):\ \lambda_1\ge0,\ \lambda_3>0,\ \min{x\in[0,1]} \quad M_U(x;\lambda_1,\lambda_3,y)\ge 0\Big\}.\]
For the lower-side ($EVaR \leq x$) we define multiplicative updates of our test supermartingale 
\[M_L(x;\gamma,z_0,x^\star)
:=1-\gamma\big(e^{-\rho+z_0 x^\star}-e^{z_0 x}\big),\]
and 
$\mathcal D(x^\star,z_0):=\Big\{\gamma\ge0:\ \min_{u\in[0,1]} M_L(u;\gamma,z_0,x^\star)\ge0\Big\}$.

For arm i with true EVaR, \( \nu_i=\EVaR_\alpha(\mu_i)\), pick any priors with full support on compact feasible sets, e.g. uniform:
$q_{1i}\ \text{on}\ {\mathcal R}_1(\nu_i)\quad\text{and}\quad
q_{2i}\ \text{on}\ \mathcal D(\nu_i,z_0)$.
For the observed samples $X^i_1,\dots,X^i_{N_i(n)}$, we define

\[
U_i(n):=\E_{(\lambda_1,\lambda_3)\sim q_{1i}}
\Bigg[\prod_{j=1}^{N_i(n)} M_U\!\big(X^i_j;\lambda_1,\lambda_3,\nu_i\big)\Bigg],
\quad
L_i(n):=\E_{\gamma\sim q_{2i}}
\Bigg[\prod_{j=1}^{N_i(n)} M_L\!\big(X^i_j;\gamma,z_0,\nu_i\big)\Bigg].
\]

Now, we condition on the past, for any fixed parameters:

$\bullet$ If $y=\nu_i$ and $t:=\lambda_1/\lambda_3>0$, the EVaR constraint \(e^{-\rho}\E[e^{tX}\mid\mathcal F_{t-1}]\le e^{ty}\) implies
\[
\E\!\left[M_U(X;\lambda_1,\lambda_3,\nu_i)\mid\mathcal F_{t-1}\right]
=1+\lambda_3\Big(1-e^{\rho-ty}\E[e^{tX}\mid\mathcal F_{t-1}]\Big)\le 1.
\]

$\bullet$ If $x^\star=\nu_i$, then \(e^{-\rho}\E[e^{z_0X}\mid\mathcal F_{t-1}]\ge e^{z_0x^\star}\) gives
\[
\E\!\left[M_L(X;\gamma,z_0,\nu_i)\mid\mathcal F_{t-1}\right]
=1-\gamma\Big(e^{-\rho+z_0x^\star}-\E[e^{z_0X}\mid\mathcal F_{t-1}]\Big)\le 1.
\]
Thus each product $\prod M_U$  and $\prod M_L$ is a nonnegative supermartingale with mean $\leq  1$.

Lemma~\ref{lem:expconcave-evar} guarantees that the dual per-sample factors are positive on compact domains and that their logarithms are exp-concave in the dual parameters. Hence we can apply the exp-concave aggregation inequality \cite[Lemma~F.1]{agrawal2021BAIcvar}:
\[
\max_{\theta\in\Theta}\sum_{t=1}^T g_t(\theta)
\;\le\;
\log \E_{\theta\sim q}\!\Bigg[\exp\!\Big(\sum_{t=1}^T g_t(\theta)\Big)\Bigg]
\;+\; d\log(T{+}1)+1.
\]
This lemma links the left-hand sides of \eqref{eq:emp-du-KLU} and \eqref{eq:emp-du-KLL} to the corresponding products of supermartingale factors.
Applying it pathwise to $\sum_j \log M_U(X^i_j;\cdot)$ with $d_U=2$ (parameters $\lambda_1,\lambda_3$), and to $\sum_j \log M_L(X^j_\ell;\cdot)$ with $d_L=1$ (parameter $\gamma$ after fixing $z_0$). We get, for all $n$,
\[
\begin{aligned}
N_i(n)\,\KLiU\big(\widehat\mu_i(n),\nu_i\big)
&\le \log U_i(n)\;+\;2\log\big(N_i(n){+}1\big)+1,\\
N_j(n)\,\KLiL\big(\widehat\mu_j(n),\nu_j\big)
&\le \log L_j(n)\;+\;\log\big(N_j(n){+}1\big)+1.
\end{aligned}
\]
Thus, we sum the two, and use $N_l(n)\le n, \text{ for } l={j, i}$:
\[
N_i(n)\,\KLiU(\widehat\mu_i(n),\nu_i)+N_j(n)\,\KLiL(\widehat\mu_j(n),\nu_j)
\;\le\;\log\!\big(U_i(n)L_j(n)\big)\;+\;\underbrace{3\log(n{+}1)+2}_{=:h(n)}.
\]
Because $U_i(n)L_j(n)$ is a nonnegative supermartingale with \(\E[U_i(n)L_j(n)]\le1\), Ville yields
$\Pr\!\Big(\exists n\ge1:\ \log(U_i(n)L_j(n))\ge x\Big)\;\le\;e^{-x}$ \cite{}.
Rearranging gives exactly our proposition:
\[
\Pr\!\Big(\exists n\ge1:\ N_i\KLiU(\widehat\mu_i,\nu_i)+N_j\KLiL(\widehat\mu_j,\nu_j)\ \ge\ x+h(n)\Big)\;\le\;e^{-x}.
\]
Apply Ville’s inequality to the mixture to obtain the stated tail bound.
Finally, take $x=\log\!\big(\tfrac{K-1}{\delta}\big)$ and union bound over $j\neq i$ to get the
$\delta$-level GLRT threshold used in the algorithmic stopping rule. 
\end{proof}

% \begin{remark}
%  The statement and proof of sample complexity (Lemma F.5 in CVaR ) only use  the C-tracking lower bound on per-arm sample counts and DKW concentration of the empirical CDFs.  They do not depend on the risk functional being CVaR or EVaR. 
%  \end{remark}

\paragraph{Stopping threshold.}
Define
\(\beta(n,\delta)\;=\;\log\!\frac{K-1}{\delta}\;+\;3\log(n+1)\;+\;2
\).
Then, by a union bound over $i\in\{2,\dots,K\}$, each alternative’s tail is $\le \delta/(K{-}1)$ and
the EVaR Track-and-Stop is $\delta$-correct.

\begin{remark}
\label{rem:EVaR-regularity}
We recall the properties we presented in the paper here:
(i) $e_\alpha:\Pc([0,1])\to[0,1]$ is continuous (Lemma~\ref{lem:evar-cont});
(ii) EVaR sub/super-level sets are compact (Lemma~\ref{lem:detailed-levelset});
(iii) $(\eta,x)\mapsto \KLiL(\eta,x)$ and $(\eta,x)\mapsto \KLiU(\eta,x)$ are continuous on
$\Pc([0,1])\times(0,1)$ (Lemma~\ref{lem:KLiLKLiU-cont});
(iv) $g_{a,j}$ and $\Phi$ are continuous in $(\nu,t)$ and $\Phi(\nu,\cdot)$ is concave on $\Sigma_K$;
(v) the argmax correspondence $t^*(\nu):=\arg\max_{t\in\Sigma_K}\Phi(\nu,t)$ is nonempty, convex,
compact-valued, and upper hemicontinuous in $\nu$ (Lemma~\ref{lem:stab_EVaR}).
\end{remark}

By continuity Remark~\ref{rem:EVaR-regularity} and the SLLN for each arm, along any sample path where
$\widehat\nu_k(n)\Rightarrow \nu_k$ for all $k$, continuity implies
$t^*(\widehat\nu(n))\to t^*(\nu)$ (upper hemicontinuity and compactness ensure cluster points lie in $t^*(\nu)$),
and the tracking rule enforces
\[
\frac{N_k(n)}{n}\ \xrightarrow[n\to\infty]{}\ t^*_k(\nu)
\qquad\text{for some }t^*(\nu)\in t^*(\nu).
\]

Fix $j\neq 1$. By the joint continuity of
\((\eta,x)\mapsto\KLiU(\eta,x),\KLiL(\eta,x)\) and the common-level reduction,
\[
\frac{1}{n}\Big[
N_j(n)\KLiU(\widehat\nu_j(n),x)\ +\
N_{1}(n)\KLiL(\widehat\nu_{1}(n),x)
\Big]
\ \xrightarrow[n\to\infty]{}\
t^*_j(\nu)\KLiU(\nu_j,x)\ +\ t^*_{1}(\nu)\KLiL(\nu_{1},x),
\]
uniformly over $x\in I_{1,j}(\nu)$ (the feasible interval is compact, endpoints move continuously).
Taking the infimum over $x$ yields
\[
\lim_{n\to\infty}\ \frac{1}{n}\,g_{1,j}(\widehat\nu(n),N(n)/n)
\ =\ g_{1,j}(\nu,t^*(\nu)).
\]
Since $\Phi(\cdot,\cdot)$ is the pointwise minimum over $j$, we get
\[
\lim_{n\to\infty}\ \frac{1}{n}\,\Phi(\widehat\nu(n),N(n)/n)
\ =\ \Phi(\nu,t^*(\nu))\ =\ V(\nu).
\]

The stopping condition is met when the left-hand side exceeds
$\beta(n,\delta)=\log\!\frac{K-1}{\delta}+h(n)$.
Therefore, for every $\varepsilon>0$,
\[
\Pr_\nu\!\Big(\tau_\delta\ >\ \tfrac{1+\varepsilon}{V(\nu)}\log(1/\delta)\ \Big)\ \longrightarrow\ 0.
\]
This yields $\limsup_{\delta\downarrow 0}\E_\nu[\tau_\delta]/\log(1/\delta)\le (1/V(\nu))(1+\varepsilon)$,
and since $\varepsilon>0$ is arbitrary, the claimed upper bound follows, hence asymptotic optimality with lower bound given in \eqref{th:GK2016} is proved.   
\end{proof}

% EVaR duals: exp-concavity checks ???Lemma F.1 applies
The next lemma establishes the conditions needed to apply \cite[Lemma~F.1]{agrawal2021BAIcvar}. The proof is given afterward.
\begin{lemma}\label{lem:expconcave-evar}
Let $X\in[0,1]$, fix $\rho>0$, and let $\nu\in(0,1)$.

\smallskip

\noindent\textbf{I)} For $(\lambda_1,\lambda_3)$ with $\lambda_3>0$, define
\[
F_U(\lambda_1,\lambda_3;x)
:=1+\lambda_3\Big(1-\exp\!\Big(\tfrac{\lambda_1}{\lambda_3}(x-\nu)+\rho\Big)\Big),\qquad x\in[0,1].
\]
Fix any window $0<t_{\min}\le t\le t_{\max}$ with $t=\lambda_1/\lambda_3$. If
\begin{equation}\label{eq:FU-positivity}
0\ \le\ \lambda_3\ \le\ \frac{1}{\,e^{\rho+t(1-\nu)}-1\,},
\end{equation}
then $F_U(\lambda_1,\lambda_3;x)>0$ for all $x\in[0,1]$. Moreover,
$(\lambda_1,\lambda_3)\mapsto F_U(\lambda_1,\lambda_3;x)$ is concave on $\{\lambda_3>0\}$ (for each fixed $x$).
Hence $g_U:=\log F_U$ is exp-concave on any compact subset of $\{\lambda_3>0\}$ where $F_U>0$.

\smallskip
\noindent\textbf{II)}
For $\beta>0$ and $\gamma\ge 0$, set $A(\beta,\nu):=e^{-\rho+\beta\nu}$ and define
\[
F_L(\gamma;\beta,\nu,x):=1+\gamma\big(e^{\beta x}-A(\beta,\nu)\big),\qquad x\in[0,1].
\]
Then $\gamma\mapsto F_L(\gamma;\beta,\nu,x)$ is affine (hence concave). The worst case in $x$ for positivity occurs at $x=0$,
so
\begin{equation}\label{eq:FL-positivity}
\min_{x\in[0,1]}F_L(\gamma;\beta,\nu,x)=1+\gamma\,(1-A(\beta,\nu)).
\end{equation}
Consequently:
\begin{itemize}
\item If $A(\beta,\nu)\le 1$ then $F_L(\gamma;\beta,\nu,x)\ge 1$ for all $\gamma\ge 0$, $x\in[0,1]$.
\item If $A(\beta,\nu)>1$ then $F_L(\gamma;\beta,\nu,x)>0$ for all $x\in[0,1]$ whenever
\[
0\ \le\ \gamma\ <\ B(\beta,\nu):=\frac{1}{\,A(\beta,\nu)-1\,}.
\]
\end{itemize}
On any compact interval where $F_L>0$ when $A>1$, or an arbitrary bounded interval if $A\le 1$),
$g_L:=\log F_L$ is exp-concave.
\end{lemma}

\begin{proof}
\emph{(I) Concavity and positivity.}
Write $a:=x-\nu\in[-1,1]$ and $t:=\lambda_1/\lambda_3$. Consider
\[
\phi(\lambda_1,\lambda_3):=\lambda_3\exp\!\Big(\tfrac{\lambda_1}{\lambda_3}a+\rho\Big)
= e^{\rho}\,\underbrace{\lambda_3\exp\!\Big(\tfrac{\lambda_1}{\lambda_3}a\Big)}_{\text{perspective of } u\mapsto e^{au}}.
\]
Since $u\mapsto e^{au}$ is convex, its perspective $(\lambda_1,\lambda_3)\mapsto \lambda_3 e^{(\lambda_1/\lambda_3)a}$ is convex
on $\{\lambda_3>0\}$ (see, e.g., Boyd--Vandenberghe, \emph{Convex Optimization}, Sec.~3.2.6).
Therefore $-(\lambda_3 e^{(\lambda_1/\lambda_3)a+\rho})$ is concave, and adding the affine terms $1+\lambda_3$ preserves concavity,
so $F_U$ is concave in $(\lambda_1,\lambda_3)$.

For positivity on $x\in[0,1]$ note that $x\mapsto \exp(\tfrac{\lambda_1}{\lambda_3}(x-\nu)+\rho)$ is nondecreasing in $x$
when $t\ge 0$, hence $x\mapsto F_U(\lambda_1,\lambda_3;x)$ attains its minimum at $x=1$. Thus
\[
\min_{x\in[0,1]} F_U(\lambda_1,\lambda_3;x)
=1+\lambda_3\big(1-e^{\rho+t(1-\nu)}\big),
\]
which is $>0$ whenever \eqref{eq:FU-positivity} holds. On any compact set where $F_U>0$, $g_U=\log F_U$ is exp-concave
by definition (since $e^{g_U}=F_U$ is concave and positive).

\smallskip
\emph{(II) Affinity and positivity.}
For fixed $(\beta,\nu)$, $F_L(\gamma;\beta,\nu,x)=1+\gamma(e^{\beta x}-A)$ is affine in $\gamma$ hence concave.
Because $x\mapsto e^{\beta x}$ is nondecreasing, for $\gamma\ge 0$ the minimum over $x\in[0,1]$ occurs at $x=0$
and equals $1+\gamma(1-A)$, proving \eqref{eq:FL-positivity}. When $A\le 1$ this is $\ge 1$ for all $\gamma\ge 0$;
when $A>1$ it remains positive on the compact interval $[0,B(\beta,\nu))$.
On any compact set with $F_L>0$, $g_L=\log F_L$ is exp-concave.
\end{proof}

% \subsection{Useful lemmas to proof Proposition~\ref{Prf:Th:upp_sample}}
% \label{app:UF_Lem_Prf_prp_upB}
% {\color{blue}\paragraph{Regions in KL-projections:}
% Using the distributionally robust representation, 
% $\KLiU$ admits a convex primal with an auxiliary $Q$, and a smooth dual in multipliers $(\lambda_1,\lambda_3)$ whose log-integrand
% \[
% g^{U}(\lambda_1,\lambda_3;X)\;=\;\log\!\Big(1+\lambda_3\big(1-\exp\!\big[\tfrac{\lambda_1}{\lambda_3}(X-y)+\rho\big]\big)\Big)
% \]
% is exp-concave over its positive region, enabling stable Newton or quasi-Newton steps.
% Via the MGF variational form, 
% the constraint $\EVaR_\alpha(\kappa)\le x$ reduces to $\exists\,\beta>0:\ \mathbb{E}_\kappa[e^{\beta X}]\le e^{-\rho+\beta x}$.
% This yields an \emph{outer} one-dimensional minimization over $\beta>0$ and an \emph{inner} convex problem with multiplier $\lambda$ and log-integrand
% \[
% g^{L}(\gamma;\beta,x,X)\;=\;\log\!\Big(1-\gamma\big(e^{-\rho+\beta x}-e^{\beta X}\big)\Big),
% \]
% optimized over the interval $D(\beta,x)=\big[0,\ 1/(e^{-\rho+\beta x}-1)\big)$ when $e^{-\rho+\beta x}>1$ (otherwise the dual is infeasible/non-finite). 
% This inner problem is concave in $\gamma$ and exp-concave in the log, so bisection or Newton on $\beta$ plus 1D line search on $\gamma$ suffices in practice.
% }

\end{document}